\documentclass[10pt,twocolumn,letterpaper]{article}

\usepackage[pagenumbers]{wacv} % To force page numbers, e.g. for an arXiv version

\definecolor{wacvblue}{rgb}{0.21,0.49,0.74}
\usepackage[pagebackref,breaklinks,colorlinks,allcolors=wacvblue]{hyperref}
\usepackage[capitalize]{cleveref}
\usepackage{multirow}
\usepackage{graphicx}
\usepackage[table,xcdraw]{xcolor}
\usepackage{hhline}
\usepackage{tabularx}
\usepackage{amsthm} 
\usepackage{amssymb}   
\usepackage{amsthm}  

\newtheorem{theorem}{Theorem}
\newtheorem{assumption}{Assumption}

\newtheorem{corollary}{Corollary}

\crefname{theorem}{Theorem}{Theorems}
\Crefname{theorem}{Theorem}{Theorems}
\crefname{assumption}{Assumption}{Assumptions}
\Crefname{assumption}{Assumption}{Assumptions}
\crefname{proposition}{Proposition}{Propositions}
\Crefname{proposition}{Proposition}{Propositions}
\crefname{corollary}{Corollary}{Corollaries}
\Crefname{corollary}{Corollary}{Corollaries}

\definecolor{lightgrey}{rgb}{0.92,0.92,0.92}
\def\wacvPaperID{863} % *** Enter the WACV Paper ID here
\def\confName{WACV}
\def\confYear{2027}

\newcommand{\rev}[1]{#1}

\title{Full Collapse Is Suboptimal: Controlled Partial Neural Collapse for Long-Tailed Recognition}

\author{
Hyeontaek Hwang\textsuperscript{\rm 1} \qquad
Cheonghwan Hur\textsuperscript{\rm 2} \qquad
Daeyoung Kim\textsuperscript{\rm 1} \\
\textsuperscript{\rm 1}KAIST\qquad
\textsuperscript{\rm 2}RTST Inc.\\
{\tt\small \textsuperscript{\rm 1}\{hwanght, kimd\}@kaist.ac.kr} \quad
{\tt\small \textsuperscript{\rm 2}princehur@rtst.co.kr}
}

\begin{document}

\maketitle
\begin{abstract}
Recent long-tailed recognition (LTR) methods incorporate neural collapse (NC) and align features to a fixed Equiangular Tight Frame (ETF) as a uniform target for all classes. However, forcing tail classes to fully collapse over-contracts their samples and weakens generalization. To address this, we propose a controlled partial neural collapse, realized through two coordinated controls on a backbone: the Class-Adaptive Sparse Activation (CASA) classifier and the ETF-constrained adaptive margin (ETF-AM) loss. Rather than forcing a uniform full collapse, this regime allows head classes to collapse along the ETF while tail classes retain within-class diversity. This prevents the tail from being forced into full collapse, which is suboptimal under imbalance. \rev{To set the dimensionality of collapse, CASA induces class-related node specialization through feature-conditioned gating so that tail classes occupy their own feature directions. For the clearance axis, ETF-AM applies a frequency-adaptive margin, with smaller margins for head classes and larger for tail classes.} On CIFAR100-LT, ImageNet-LT, Places-LT, and iNaturalist-2018, our method achieves the best overall accuracy over prior methods. In the empirical analysis, within-class spread increases monotonically from head to tail classes, consistent with tail classes retaining variability rather than fully collapsing.
\end{abstract}

\section{Introduction}
\label{sec:intro}

Real-world data often follows a long-tailed distribution, where head (Many) classes hold most samples, and tail (Few) classes hold only a handful~\cite{peng2025bmb,zhang2021distribution}. 
Models trained on such data are biased toward head classes and lose accuracy on tail classes, which limits their reliable deployment in safety-critical domains~\cite{gharaee2023step,wang2022parallel,zheng2024large}. 
Long-tailed recognition (LTR) studies how to train accurate models despite this imbalance~\cite{cao2019learning,cui2019class,hong2021disentangling,li2022long,lin2017focal,long2022retrieval,pan2021model,ren2020balanced,shi2023resampling,yun2019cutmix,zhang2024dynamic,zhang2018mixup,zhao2022adaptive}.

A recent line of LTR research draws on neural collapse (NC), which describes the terminal-phase geometry of a network trained on balanced data~\cite{papyan2020prevalence}. 
Under NC, within-class features collapse to their class mean, and the class means form a simplex Equiangular Tight Frame (ETF) that maximizes inter-class separation.
Several methods import this geometry into LTR by fixing the classifier to an ETF and training the features to match it~\cite{liu2023inducing,yang2022inducing} or by regularizing feature directions toward it~\cite{hasegawa2024exploring}. 
These methods pursue \emph{full} collapse for every class.

Full collapse, however, is the optimal geometry only for balanced data. 
Under imbalance, the global-optimal last-layer geometry deviates significantly from the balanced ETF.
For instance, minority classes may degenerate into a single shared direction~\cite{fang2021exploring} or, depending on the loss, scatter across multiple directions~\cite{yan2024neural}.
Additionally, finite regularization further distorts this structure away from the ideal ETF~\cite{thrampoulidis2022imbalance,dang2024neural}.

Forcing tail classes to collapse fully therefore over-contracts their samples and removes the within-class variance that tail generalization needs~\cite{hui2022limitations,galanti2022on}. 
Prior NC-based LTR methods inherit this problem because they target the same full collapse geometry across all classes.

We reframe LTR around \emph{controlled partial neural collapse}: head classes effectively collapse along the ETF while tail classes retain within-class diversity. 
The two requirements act on distinct NC properties, within-class collapse (NC1) and inter-class separation (NC2), which the standard NC framework quantifies separately~\cite{papyan2020prevalence}. 
The two interact through a trade-off that favors retaining a positive tail spread over full collapse.

\begin{figure*}[t]
  \centering
  \includegraphics[width=\textwidth]{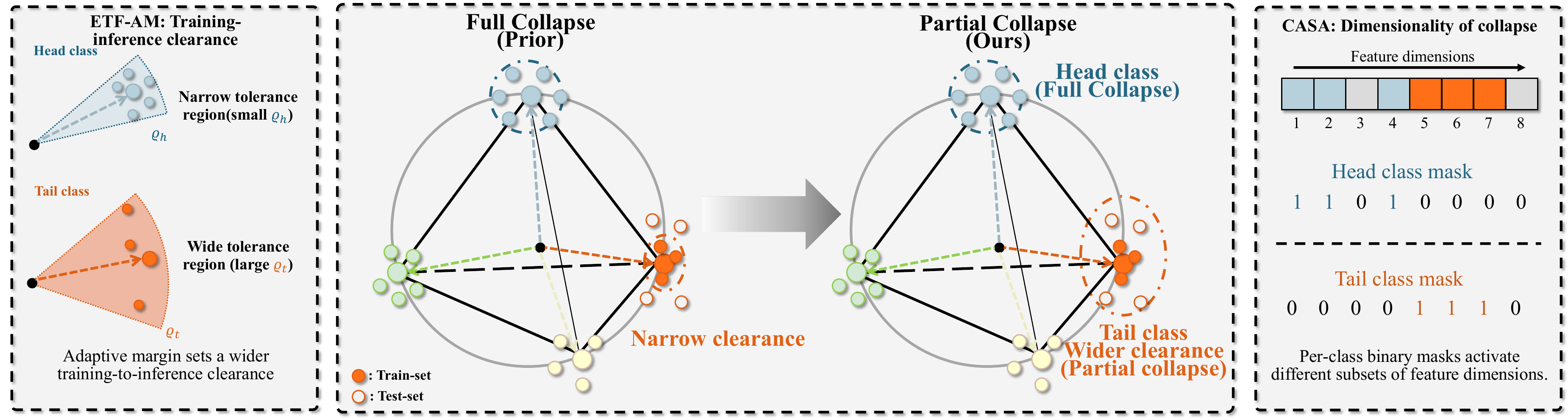}

  \caption{Controlled partial neural collapse via two coordinated mechanisms. Full collapse over-concentrates tail classes at their ETF vertices (center); our method addresses this by controlling the dimensionality of collapse via instance-conditioned node selection (CASA, right; \cref{sec:casa}) and the training-to-inference clearance per class (ETF-AM, left; \cref{sec:etfam}).}
  \label{fig:concept}
\end{figure*}

Based on this motivation, we prove that a positive within-class spread, rather than full collapse, lowers balanced test risk under a finite-sample class-conditional model. 
We realize controlled partial collapse with two coordinated controls on a backbone (\cref{fig:concept}). 
One control sets the \emph{dimensionality} of collapse. 
\rev{CASA uses feature-conditioned node-level gating, allowing different classes to develop distinct activation patterns.}
Different classes then occupy distinct feature directions instead of a single shared subspace.
\rev{The other control targets the training-to-inference clearance. Our ETF-constraint Adaptive Margin (ETF-AM) loss fixes the classifier to an ETF for inter-class separation and assigns larger margins to tail classes, widening their training-to-inference clearance.}

In summary, our contributions are as follows:
\begin{itemize}
    \rev{\item \textbf{Theoretical suboptimality of full collapse:} Under a finite-sample class-conditional model, we show that full collapse is not risk-optimal for classes below an explicit sample-count threshold, and that the optimal within-class spread is ordered by class frequency among those classes.}\rev{\item \textbf{Controlled Partial Collapse:} To effectively induce partial collapse, we propose two coordinated mechanisms: the CASA classifier gates feature nodes to dictate the \emph{dimensionality} of collapse, while the ETF-AM loss imposes an adaptive margin to set its clearance.}
    \item \textbf{Empirical validation:} Our method achieves the best overall accuracy over prior methods on CIFAR100-LT, ImageNet-LT, Places-LT, and iNaturalist-2018. The empirical analysis shows that further tail compression degrades generalization on tail classes.
\end{itemize}

\section{Background and Related Work}
\label{sec:related}

Existing approaches to LTR have primarily relied on re-weighting~\cite{guo2022learning,lan2025mixture,rangwani2022escaping,ren2020balanced,tan2021equalization,zhu2022balanced}, loss design~\cite{cui2019class,wu2020distribution}, logit adjustment~\cite{lan2025mixture,menon2021long,zhang2024dynamic,zhao2022adaptive}, and margin adjustment~\cite{cao2019learning,kang2021learning}. 
These methods apply margins within a learnable, unconstrained classifier~\cite{alshammari2022long,cao2019learning} and therefore lack an explicit geometric constraint on the final embedding structure~\cite{yang2022inducing}.

Several studies induce NC to improve long-tailed recognition~\cite{hasegawa2024exploring,liu2023inducing,yang2022inducing}. NC~\cite{papyan2020prevalence} emerges at the end of training on balanced data: within-class features collapse to one center (NC1), class centers form a simplex ETF (NC2), classifier weights align with the centers (NC3), and decisions follow the nearest center (NC4). 
These methods fix the classifier to an ETF, or regularize features toward it, and pursue full collapse for every class.

Theoretical analyses~\cite{fang2021exploring,thrampoulidis2022imbalance,dang2024neural, yan2024neural} show that the optimal geometry deviates from that of the balanced ETF when the data is imbalanced. 
The layer-peeled model exhibits minority collapse, in which minority-class weights degenerate into a single shared direction beyond an imbalance threshold~\cite{fang2021exploring}. 
Vanishing-regularization analysis shows the optimum follows a simplex-encoded-label geometry rather than the ETF~\cite{thrampoulidis2022imbalance}. 
Similarly, an unconstrained ReLU analysis confirms that class means become orthogonal with class-dependent lengths rather than equiangular~\cite{dang2024neural}. 
Recent work argues that minority classes are best represented by multiple centers rather than a single collapsed mean~\cite{yan2024neural}.

A complementary line of work studies inter-class geometry when the number of classes $C$ exceeds the feature dimension $d+1$. 
In this regime, a simplex ETF is geometrically impossible; prior work establishes that the optimal inter-class target is instead a \emph{generalized maximal-separation frame} with coherence approaching the Welch bound $\sqrt{(C-d)/(d(C-1))}$~\cite{gao2023grassmannian,jiang2024generalized,liu2023generalizing}. 
These works target the inter-class frame under full collapse. Our analysis is orthogonal to this choice, since it acts only on the within-class spread and holds for any maximal-separation target.

Recent methods continue to pursue \emph{full} NC recovery: All-around Neural Collapse (AllNC)~\cite{zhang2024all} restores collapse across activations, means, and classifier vectors. 
In parallel, Space Alignment~\cite{wang2025space} corrects feature-classifier misalignment to better reach the ETF. 
Both treat the balanced ETF as the target for every class. 
We take the complementary view: for tail classes, this target benefits from relaxation rather than full collapse. 
Moreover, achieving such full NC recovery with these methods relies on heavier training machinery, such as additional branches, contrastive objectives, or alignment modules, rather than on a single backbone.

\section{Methodology}
\label{sec:method}

Prior NC-based LTR methods operate on the inter-class axis, refining the mean geometry while keeping every class fully collapsed. 
In contrast, we achieve controlled partial collapse by intervening on the
within-class geometry along two distinct axes. On the \emph{dimensionality} axis, the CASA classifier (\cref{sec:casa}) uses node-level gating to give each class its own feature subspace rather than forcing all classes to share one.
On the \emph{clearance} axis, the ETF-AM loss (\cref{sec:etfam}) applies a per-class adaptive margin that sets a frequency-dependent training-to-inference clearance.
In this section, we first present the theoretical motivation for this design, then describe the main components, as illustrated in \cref{fig:concept}.

\paragraph{Theoretical Motivation.}
We utilize the unconstrained-features model (UFM), which treats last-layer features as free variables and is standard for NC analysis~\cite{dang2024neural,yang2022inducing}.
We consider $C$ classes with sample counts $n_1 \geq \cdots \geq n_C$. 
The class means align with a maximal-separation inter-class target, yielding a base separation margin $\gamma>0$ from the nearest decision boundary. 
% Our two controls act on the within-class geometry, and the partial-collapse analysis assumes only this geometry, so it holds regardless of the inter-class target. 

The within-class spread $s_c=\frac{1}{n_c}\sum_{i\in c}\|h_i-\mu_c\|^2$ measures how tightly features $h_i$ cluster around the class mean $\mu_c$; full collapse drives every $s_c$ to zero.
To reason about generalization, we adopt a finite-sample class-conditional model, defining the signed test margin of a class-$c$ sample as
\begin{equation}
\mathcal{M}_c(s_c) = (\gamma - a s_c) + Z_{\mathrm{test}} + Z_{\mathrm{est}}(s_c, n_c),
\label{eq:margin-model}
\end{equation}
\rev{with separation cost $-a s_c$, where $a>0$ is the signed distance lost per unit of spread, irreducible noise $Z_{\mathrm{test}}\sim\mathcal{N}(0,\sigma^2)$ with $\sigma^2>0$, and Gaussian finite-sample term $Z_{\mathrm{est}}$.} 
A sample is correctly classified when $\mathcal{M}_c(s_c)>0$, and $\mathcal{R}_c(s_c)=\Pr[\mathcal{M}_c(s_c)\le 0]$ is the per-class risk. 

On free features, the UFM optimum is full collapse.
But the UFM treats features as free variables with no sample count, so its optimum is independent of $n_c$ and cannot reflect how Few-sample classes generalize.

\rev{To handle this, we draw on recent transfer learning and generalization analyses~\cite{galanti2022on,galanti2022generalization}, which establish within-class variability as a key determinant of finite-sample error, with estimates that degrade as the number of available samples shrinks. We formalize this via the variance of the finite-sample term, which grows as $n_c$ decreases: raising $s_c$ from zero trades off separation against estimation variance.}
\rev{
\begin{assumption}[Finite-sample variance model]
\label{asm:reg}
Let $V(s_c,n_c)\triangleq\mathrm{Var}(Z_{\mathrm{est}})$, then
\begin{equation}
    V(s_c,n_c)=\phi(s_c)/n_c,
\end{equation}
where $\phi\ge0$ is differentiable on $s_c\ge0$ with $0<\phi(0)<\infty$ and $\phi'(0)<0$.
For the cross-class statement below, $a$, $\gamma$, $\sigma$, and $\phi$ are shared across classes; the class counts $n_c$ are the varying model parameters.

\end{assumption}
Since $\mathcal{M}_c$ is Gaussian, $\mathcal{R}_c(s_c)=\Phi(-g(s_c))$ with $g=(\gamma-as_c)/\sqrt{\sigma^{2}+V(s_c,n_c)}$. The condition $\mathcal{R}_c'(0)<0$, under which a small positive spread lowers the risk, holds exactly when
\begin{equation}
\label{eq:sensitivity}
-V_s(0,n_c) > \frac{2a\left[\sigma^{2}+V(0,n_c)\right]}{\gamma} ,
\end{equation}
where $V_s=\frac{\partial V}{\partial s_c}$. We call this the \emph{sensitivity condition} (\cref{sec:supp-ufm}).
}

\rev{\begin{theorem}[Suboptimality of full collapse]
\label{thm:partial}
Let $V$ be as in \cref{asm:reg}. For any class $c$ satisfying \cref{eq:sensitivity}, the per-class risk on $[0,\gamma/a]$ attains its global minimum in the interior. Writing $s_c^\star$ for the smallest global minimizer,
\begin{equation}
s_c^\star>0 \quad\text{and}\quad \mathcal{R}_c(0)>\mathcal{R}_c(s_c^\star).
\end{equation}
\end{theorem}}

\rev{Because the balanced risk is the average of the separable class-wise risks, if at least one class satisfies \cref{eq:sensitivity}, the all-class zero-spread solution is not globally optimal within this model. \cref{thm:partial} leaves open which classes satisfy the condition; the following \cref{cor:scarcity} resolves it into a threshold on the training count.}

\rev{
\begin{corollary}[Scarcity-dependent spread]
    \label{cor:scarcity}
    Let $n^\star$ be the sample count at which \cref{eq:sensitivity} changes sign. Then \cref{eq:sensitivity} holds if and only if $0<n_c<n^\star$. If, in addition, $\phi'(s)<0$ for every $s\in(0,\gamma/a)$, then for any two classes $c,c'$ with $n_{c'}<n_c<n^\star$,
    \begin{equation}
    s_c^{\star}<s_{c'}^{\star}.
    \end{equation}
\end{corollary}
}

\rev{The threshold is what \cref{eq:sensitivity} becomes once $V=\phi/n_c$ is substituted, and it is non-vacuous when $-\gamma\phi'(0)>2a\phi(0)$ (\cref{sec:supp-ufm-thm}). The ordering needs $\phi$ to decrease on the whole interval not only at zero. Within the shared-parameter model, both the threshold and the ordering are governed by $n_c$.}

\subsection{Class-Adaptive Sparse Activation Classifier}
\label{sec:casa}

\begin{figure}[t]
  \centering
  \includegraphics[width=0.95\linewidth]{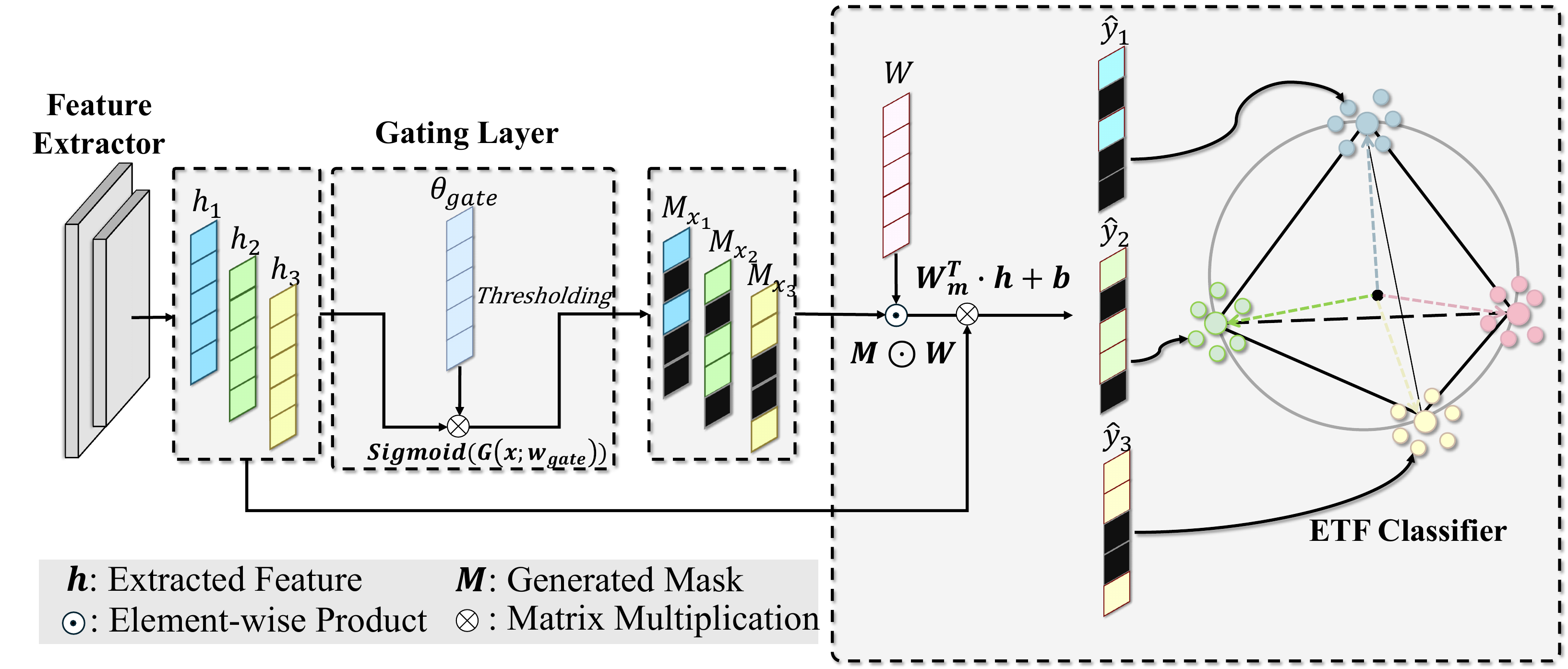}
  \caption{The overall architecture of the CASA classifier, which selects a subset of nodes through instance-conditioned node-level gating.}
  \label{fig:arch}
\end{figure}

\subsubsection{Detailed Architecture}
The CASA classifier controls the \emph{dimensionality} of collapse.
Forcing every class to collapse within the same shared feature subspace over-contracts tail classes, since their features are squeezed into a subspace shared across classes rather than their own dedicated direction~\cite{yan2024neural}.
As shown in~\cref{fig:arch}, CASA instead assigns each class its own subset of feature dimensions through a node-level gating layer, so head and tail classes share the fixed ETF target for their means but operate in distinct subspaces around it.

CASA incorporates a node-level gating layer, a single fully connected layer integrated within the main classification layer, which selectively activates a subset of specialized nodes for each instance (\cref{eq:softmask}):
\begin{equation}
\label{eq:softmask}
M_{\text{soft}} = \text{Sigmoid}(G(h;\,\omega_{\text{gate}})),
\end{equation}
where $M_{\text{soft}} \in \mathbb{R}^d$ represents the soft mask vector, $G(\cdot)$ denotes the gating network, $h = f(x;\,\omega)$ is the feature produced by the feature extractor for input $x$, and $\omega_{\text{gate}}$ are the trainable parameters of the gating layer.

During training, $M_{\text{soft}}$ is discretized using a thresholding operation to produce the binary mask $M$ for class-related node selection:
\begin{equation}
\label{eq:mask}
M = \begin{cases} \mathbb{I}_{\{m_i > 0.5\}}(M_{\text{soft}}) & \text{if training} \\ M_{\text{soft}} & \text{otherwise} \end{cases},
\end{equation}
where $m_i$ is the $i$-th element of $M_{\text{soft}}$ and $\mathbb{I}$ is the indicator function. 
The hard threshold during training drives class-specific node selection; the soft mask at test aggregates node contributions smoothly for stable generalization.

The mask is applied along the node dimension, $W_m = M \odot W$, and the forward pass is $\hat{y} = W_m^\top h + b$. 
Since $M$ acts on the node dimension shared by all classes, this is equivalent to masking the feature: 
\begin{equation}
\label{eq:masked-feat}
    W_m^\top h = W^\top (M \odot h),
\end{equation}
where the classifier weights are fixed to the ETF structure. 
CASA selects the subspace in which each class's features collapse.

\rev{\subsubsection{Instance-Conditioned Gating and Its Class Structure}}
\label{sec:casa-gating}
\rev{The gate is conditioned on the instance feature $h$ and trained jointly with the classifier.}
When the masked weights on the classifier are $W_m = M \odot W$, the prediction is
\begin{equation}
\label{eq:cls-gate}
\hat{y} = f_{\text{cls}}(h;W_m),
\end{equation}
where $\hat{y}$ is the prediction and $f_{\text{cls}}(\cdot)$ is that classifier. 
\rev{Optimizing it with a classification loss such as cross-entropy gives:
\begin{equation}
\label{eq:lcls}
\mathcal{L}_{\text{cls}} = -\sum_{i=1}^{N} y_i^\top \log f_{\text{cls}}(h_i; W_m),
\end{equation}
where $y_i$ is the one-hot ground-truth label and $h_i$ is the feature of sample $x_i$.}
Because the mask is produced from the feature and the loss is minimized class-wise, samples of the same class receive correlated gate inputs, which empirically yields similar masks within a class.

During the backward pass, we use a straight-through estimator \cite{bengio2013estimating,srinivas2017training}. 
\rev{The hard threshold is bypassed on the backward path, and gradients flow only through the weights of activated nodes (\cref{eq:bwd}):}
\begin{equation}
\label{eq:bwd}
\frac{\partial \mathcal{L}}{\partial \omega_{\text{gate}}} = \mathbb{I}_{\{m_i > 0.5\}} \cdot \frac{\partial \mathcal{L}}{\partial m_i} \cdot \frac{\partial m_i}{\partial \omega_{\text{gate}}}.
\end{equation}

\rev{This selective gradient flow restricts updates to the activated nodes.}
Furthermore, we employ a spike-and-slab prior~\cite{srinivas2017training} on the gating mechanism to encourage sparse masks. The prior term is formulated as:
\begin{equation}
\label{eq:ssprior}
\mathcal{L}_g = \vartheta(1 - m) + (1 - \vartheta)\frac{m^2}{2\varpi^2}.
\end{equation}

Here $\vartheta$ weighs the spike-and-slab term and $\varpi^2$ is the slab variance: the spike term pushes nodes toward activation while the slab term prevents the gate from saturating, balancing sparsity against signal.
Unlike attention-based or feature reweighting methods, \rev{the hard threshold and sparse gradient flow activate a sparse subset of nodes rather than reweighting all nodes.}

\subsection{ETF-Constrained Adaptive Margin Loss}
\label{sec:etfam}
\rev{Under the additional monotonicity condition, \cref{cor:scarcity} gives a wider risk-optimal spread to scarcer classes. Following this ordinal prediction, the ETF-AM loss combines a fixed ETF classifier with a per-class adaptive margin, giving scarcer classes a wider training-to-inference clearance.}

\subsubsection{ETF-Constrained Margin Formulation}
\label{sec:etfam-margin}
Consider a multi-class classification problem with $C$ classes, where the number of samples in class $c$ is denoted by $n_c$. 
We define an adaptive margin for each class as:
\begin{equation}
\label{eq:margin}
\varrho_c = \frac{\Gamma}{n_c^\varsigma},
\end{equation}
where $\Gamma > 0$ is the margin scale fixed by the ETF geometry below. \rev{We set $\varsigma = 0.25$ following the frequency scaling of LDAM~\cite{cao2019learning}. The exponent is a design choice rather than a rate derived from \cref{thm:partial,cor:scarcity}.}
% The maximum margin is determined by the ETF geometry. 
Under the ETF constraint, the Euclidean distance between any two distinct class weight vectors is
\begin{equation}
\label{eq:dist}
\|w_c - w_{c'}\| = \sqrt{2 + \tfrac{2}{C - 1}}, \quad c \neq c'.
\end{equation}
Placing the decision boundary at the midpoint yields 
\begin{equation}
\label{eq:maxmargin}
\Gamma = \tfrac{1}{2}\|w_c - w_{c'}\| \to \tfrac{\sqrt{2}}{2} \text{ as } C \to \infty,
\end{equation}
with the full derivation in \cref{sec:supp-margin} of the supplementary material.
\rev{Tail classes receive a larger $\varrho_c$ and correspondingly larger clearance, while head classes keep a small margin.}

We adopt an ETF classifier that enforces the constraints:
\begin{equation}
\label{eq:etf}
\|w_c\| = 1 \quad \text{and} \quad \langle w_c, w_{c'} \rangle = -\frac{1}{C - 1} \quad \text{for } c \neq c',
\end{equation}
which guarantees that all class weight vectors $w_c$ are uniformly distributed on the unit hypersphere.
Given a feature embedding $h \in \mathbb{R}^d$ extracted by a feature extractor $f(x;\,\omega)$, the logit for class $c$ is computed as:
\rev{\begin{equation}
\label{eq:logit}
% z_c = w_c^\top h,
z_c = w_c^\top u + b_c, \quad u = M \odot h,
\end{equation}}
\rev{where $u$ is the gated feature of \cref{eq:masked-feat}.}
We then modify the logit for the true class $y$ by subtracting the corresponding margin:
\begin{equation}
\label{eq:adjlogit}
\hat{z}_c = \begin{cases} z_c - \varrho_c, & \text{if } c = y, \\ z_c, & \text{otherwise} \end{cases}.
\end{equation}
\rev{Subtracting $\varrho_y$ makes the training condition $z_y-z_k>\varrho_y$ stricter than the inference rule $z_y-z_k>0$. The two hyperplanes are parallel and separated by $\Delta_{yk}=\varrho_y/\lVert w_y-w_k\rVert$. By \cref{eq:dist}, this distance is independent of the competitor $k$, giving a single per-class clearance $\Delta_y=\varrho_y/(2\Gamma)$. A feature that meets the training condition therefore retains at least $\Delta_y$ in signed distance once the margin is removed. The margin is imposed on the gated feature, and the ETF fixes $\lVert w_y-w_k\rVert$ irrespective of the mask, so gating leaves $\Delta_y$ unchanged.}

\rev{While the frequency scaling follows LDAM~\cite{cao2019learning}, the fixed frame changes what the margin means geometrically. Under a learnable head the same $\varrho_y$ maps to a different distance for each competitor $k$, whereas here it maps to a single per-class clearance $\Delta_y$.}
The margin-adjusted loss is defined via the standard cross-entropy formulation:
\begin{equation}
\label{eq:etfam}
\mathcal{L}_{\text{ETF-AM}} = -\frac{1}{N}\sum_{i=1}^{N} \log \frac{\exp(z_{y_i} - \varrho_{y_i})}{\exp(z_{y_i} - \varrho_{y_i}) + \sum_{c \neq y_i}\exp(z_c)},
\end{equation}
where $N$ is the number of samples, $y_i$ denotes the ground-truth label for sample $i$ and $\varrho_{y_i}$ is the adaptive margin for that class.
Because the ETF weights are fixed at maximal separation (\cref{eq:etf}), the network can reduce $\mathcal{L}_{\text{ETF-AM}}$  only through the backbone and the gate, not by rotating the classifier.
\rev{A further discussion and clearance derivation are in \cref{sec:supp-clearance}.}

The overall loss function integrates the classification loss and the sparse regularization term induced by the spike-and-slab prior:
\rev{\begin{equation}
\label{eq:total}
\mathcal{L} = \mathcal{L}_{\text{ETF-AM}} + \mathcal{L}_g + \beta \cdot\lVert h\rVert^2,
\end{equation}
where $\mathcal{L}_g$ corresponds to the sparse regularization in \cref{eq:ssprior}, $\beta$ is the coefficient of the feature-norm regularizer.}
% The overall loss combines the ETF-AM loss with the sparse regularization term.
Together, the $\mathcal{L}_{\text{ETF-AM}}$ and $\mathcal{L}_g$ terms act on the dimensionality of collapse and clearance, thereby promoting partial collapse.

\section{Experiments}
\label{sec:exp}

\subsection{Experiment Setup}

\subsubsection{Metric}
We use top-1 accuracy~\cite{alshammari2022long} and, for detailed analysis, the standard split into Many ($> 100$ samples), Medium ($> 20$ and $\leq 100$), and Few ($\leq 20$) subsets, following previous works~\cite{alshammari2022long,cao2019learning,liu2019large}.

Following Hasegawa and Sato~\cite{hasegawa2024exploring}, we also use the FDR score~\cite{fisher1936use}, the ratio of between-class scatter to within-class scatter, which quantifies the emergence of NC; mask-consistency metrics (cosine similarity~\cite{mikolov2013efficient}, symmetric KL~\cite{kullback1951information}) are defined in \cref{sec:supp-metrics} of the supplementary material. In our partial collapse setting, FDR primarily tracks inter-class separation (NC2/NC3); within-class relaxation is measured separately via the per-class spread of~\cref{sec:partial-evidence}.

\subsubsection{Dataset}
We use four datasets to evaluate the proposed method on long-tailed classification tasks.

\textbf{CIFAR100-LT.} CIFAR100-LT down-samples CIFAR100~\cite{krizhevsky2009learning} with imbalance factor $\rho=\text{max}(n_c)/\text{min}(n_c)$, the ratio between the largest and smallest class sizes~\cite{alshammari2022long}.

\textbf{ImageNet-LT.} Following~\cite{liu2019large}, we construct ImageNet-LT by subsampling the original ImageNet~\cite{deng2009imagenet} using a Pareto distribution with a power value $\alpha = 6$~\cite{arnold2014pareto}, resulting in 5 to 1{,}280 images per class for training.

\textbf{Places-LT.} Following the same procedure, Places-LT applies a Pareto distribution ($\alpha = 6$) to the 365 classes of Places~\cite{zhou2017places}, with class sizes from 5 to 4{,}980.

\textbf{iNaturalist-2018.} iNaturalist-2018~\cite{inaturalist} is a naturally long-tailed fine-grained dataset (imbalance factor $\rho$ of $500$) with $C=8{,}142$ classes, 437K training and 24K validation images.

\subsubsection{Implementation Details}
We use ResNet34~\cite{he2016deep} for CIFAR100-LT, ResNeXt50~\cite{xie2017aggregated} for ImageNet-LT, and ResNet152~\cite{he2016deep} for Places-LT, following the implementation of \citet{alshammari2022long,hasegawa2024exploring,zhang2022self}.
%(the ETF method of~\cite{hasegawa2024exploring} uses ResNet-34, as does the official implementation of WB~\cite{alshammari2022long}); the remaining CIFAR baselines use the smaller CIFAR-style ResNet-32, as noted in \cref{tab:cifar}. 
%We use ResNeXt50~\cite{xie2017aggregated} for ImageNet-LT and ResNet152~\cite{he2016deep} for Places-LT, following~\citet{zhang2022self}. 
For iNaturalist-2018, we use ResNet50~\cite{he2016deep} ($d=2{,}048$); since $C=8{,}142>d+1$, \rev{no simplex ETF exists, so we train the classifier directly on the features while keeping CASA gating and the ETF-AM loss with tuned $\Gamma$.}
Models are trained on long-tailed sets and evaluated on balanced test sets. 
Following prior works~\cite{liu2019large,zhang2022self,zhao2023mdcs}, Places-LT uses ImageNet-1k pre-trained weights from PyTorch~\cite{paszke2017automatic}; CIFAR100-LT and ImageNet-LT train from scratch.
We apply post-training logit adjustment (LA)~\cite{hasegawa2024exploring,kim2020adjusting}. Main results (\crefrange{tab:cifar}{tab:inat}) use a fixed seed, following the protocol of the cited baselines.
Further training settings are in \cref{sec:supp-train} of the supplementary material. 
% Code and configurations will be released upon acceptance.

\begin{table}[t]
\centering\footnotesize
\caption{Top-1 accuracy on the test set of CIFAR100-LT ($\rho=100$). Baseline numbers are taken from the original papers; ``-'' denotes results not reported in the original papers. \textbf{Bold} marks the best and \underline{underline} the second-best result. \rev{$\dagger$ indicates result with ResNet34}}
\label{tab:cifar}
  \begin{tabular*}{\linewidth}{@{\extracolsep{\fill}}l|ccc|c}
\toprule
Method & Many & Medium & Few & All \\
\midrule
CE (Baseline) & 64.60 & 36.90 & 11.90 & 38.50 \\
CB~\cite{cui2019class}           & --    & --    & --    & 39.60 \\
LDAM~\cite{cao2019learning}         & --    & --    & --    & 42.04 \\
BALMS~\cite{ren2020balanced}       & --    & --    & --    & 50.80 \\
VS+SAM~\cite{rangwani2022escaping}      & 56.40 & 48.80 & \underline{31.70} & 46.60 \\
GCL~\cite{li2022long}         & --    & --    & --    & 48.70 \\
WB~\cite{alshammari2022long} $\dagger$                    & 72.60 & 51.86 & \textbf{32.63} & \underline{53.35} \\
RBL \cite{gao2023feature} & -- & -- & -- & 53.10 \\
NC-DRW-cRT \cite{liu2023inducing}  & -- & -- & -- & 48.70 \\
DLLA~\cite{zhang2024dynamic}        & --    & --    & --    & 50.08 \\
ETF~\cite{hasegawa2024exploring} $\dagger$                   & \underline{74.20} & \underline{52.90} & 29.90 & 53.00 \\
BIB~\cite{lan2025mixture}        & --    & --    & --    & 44.90 \\
MBIB~\cite{lan2025mixture}        & --    & --    & --    & 47.50 \\
\textbf{Ours} $\dagger$                        & \textbf{74.71} & \textbf{55.77} & 31.65 & \textbf{54.73} \\
\bottomrule
\end{tabular*}
\end{table}

\begin{table}[t]
\centering\footnotesize
\caption{Top-1 accuracy on the test set of ImageNet-LT with ResNeXt50. Baseline numbers are taken from the original papers; ``–'' denotes results not reported in the original papers. \textbf{Bold} marks the best and \underline{underline} the second-best result.}
\label{tab:imagenet}
  \begin{tabular*}{\linewidth}{@{\extracolsep{\fill}}l|ccc|c}
\toprule
Method & Many & Medium & Few & All \\
\midrule
CE (Baseline)                       & \underline{65.90} & 37.50 & 7.70  & 44.40 \\
LDAM~\cite{cao2019learning}                    & --    & --    & --    & 51.30 \\
ALA~\cite{zhao2022adaptive}                    & 64.10 & 49.90 & 34.70 & 53.30 \\
WB~\cite{alshammari2022long}                      & 62.50 & 50.40 & \textbf{41.50} & 53.90 \\
RBL \cite{gao2023feature} & 64.80 & 49.60 & 34.20 & 53.30 \\
NC-DRW-cRT \cite{liu2023inducing} & 65.60 & 51.20 & 35.40 & 54.20 \\
DLLA~\cite{zhang2024dynamic}                   & --    & --    & --    & 54.50 \\
ETF~\cite{hasegawa2024exploring}                    & 62.40 & 49.30 & 36.00 & 52.60 \\
BIB~\cite{lan2025mixture}                   & 64.70 & 51.20 & 32.70 & 53.90 \\
MBIB~\cite{lan2025mixture}                   & \textbf{67.00} & \underline{52.80} & 33.50 & \underline{55.70} \\
\textbf{Ours}                       & 65.20 & \textbf{55.35} & \underline{38.87} & \textbf{56.90} \\
\bottomrule
\end{tabular*}
\end{table}
\subsection{SOTA Benchmark}

\cref{tab:cifar,tab:imagenet,tab:places} show that our method achieves the best overall accuracy over prior approaches on all three benchmarks: 54.73\% (+1.38\%p) on CIFAR100-LT, 56.90\% (+1.20\%p) on ImageNet-LT, and 41.20\% (+0.56\%p) on Places-LT. 
The gains are concentrated in the Medium and Few splits, while Many-class accuracy is preserved. These results are consistent with the frequency-adaptive margin, which assigns larger margins to underrepresented classes.

\paragraph{iNaturalist-2018.}
\cref{tab:inat} reports results on iNaturalist-2018, the natural fine-grained long-tail benchmark. Our method achieves 73.50\% overall top-1 accuracy, outperforming prior methods. 
% This suggests that the partial-collapse design transfers to a natural long-tail with $\rho \approx 500$.
We utilize a learnable classifier because no simplex ETF exists. 
The partial-collapse still applies since \cref{thm:partial} is agnostic to the inter-class target, and tail classes retain larger within-class spread than head classes here as well (\cref{tab:spread}). This will be discussed more in~\cref {sec:partial-evidence}.

\paragraph{Footprint.} 
Our method adds only a node-level gating layer. 
The gating layer and classifier scale with the embedding dimension $d$, adding 0.3M parameters at $d{=}512$ (CIFAR100-LT) and up to 4.2M at $d{=}2{,}048$ (ImageNet-LT, Places-LT, iNaturalist), while FLOPs increase by under 0.1\% across all settings as shown in \cref{tab:footprint}.

\begin{table}[t]
\centering\footnotesize
\caption{Top-1 accuracy on the test set of Places-LT with ResNet152. Baseline numbers are taken from the original papers; ``–'' denotes results not reported in the original papers. \textbf{Bold} marks the best and \underline{underline} the second-best result.}
\label{tab:places}
  \begin{tabular*}{\linewidth}{@{\extracolsep{\fill}}l|ccc|c}
\toprule
Method & Many & Medium & Few & All \\
\midrule
CE (Baseline)                       & \textbf{45.70} & 27.30 & 8.20  & 30.20 \\
OLTR~\cite{liu2019large}                   & \underline{44.70} & 37.00 & 25.30 & 35.90 \\
D-LWS~\cite{kang2020decoupling}                  & 40.60 & 39.10 & 28.60 & 37.60 \\
BALMS~\cite{ren2020balanced}                  & 41.20 & 39.80 & 31.60 & 38.70 \\
DisAlign~\cite{zhang2021distribution}               & 40.40 & \textbf{42.40} & 30.10 & 39.30 \\
LADE~\cite{hong2021disentangling}                   & 42.80 & 39.00 & 31.20 & 38.80 \\
GCL~\cite{li2022long}                    & --    & --    & --    & \underline{40.64} \\
RBL \cite{gao2023feature} & 44.10 & 40.70& 24.40 & 38.80 \\
DLLA~\cite{zhang2024dynamic}                   & 43.90 & 40.10 & \underline{32.90} & 40.10 \\
\textbf{Ours}                       & 44.10 & \underline{41.30} & \textbf{35.50} & \textbf{41.20} \\
\bottomrule
\end{tabular*}
\end{table}

\begin{table}[t]
\centering\footnotesize
\caption{Top-1 accuracy on iNaturalist-2018 with ResNet50. Baseline numbers are taken from the original papers; ``–'' denotes results not reported in the original papers. \textbf{Bold} marks the best and \underline{underline} the second-best result.}
\label{tab:inat}
  \begin{tabular*}{\linewidth}{@{\extracolsep{\fill}}l|ccc|c}
\toprule
Method & Many & Medium & Few & All \\
\midrule
CE (Baseline) & \textbf{75.70} & 66.90 & 61.70 & 65.80 \\
cRT~\cite{kang2020decoupling}          & --   & --   & --   & 65.20 \\
LWS~\cite{kang2020decoupling}          & --   & --   & --   & 65.90 \\
MiSLAS~\cite{zhong2021improving} & --   & --   & --   & 71.60 \\
GCL~\cite{li2022long}          & --   & --   & --   & 72.00 \\
WB~\cite{alshammari2022long} & 71.00 & 70.30 & 69.40 & 70.00 \\ 
DLLA~\cite{zhang2024dynamic} & -- & -- & -- & 71.30 \\
BIB~\cite{lan2025mixture} & 73.90 & 72.90 & 72.10 & 72.70 \\
MBIB~\cite{lan2025mixture} & 72.60 & \underline{73.60} & \underline{73.10} & \underline{73.30} \\
\textbf{Ours} \rev{\scriptsize{(learnable head)}}             & \underline{74.00} & \textbf{73.90} & \textbf{73.30} & \textbf{73.50} \\
\bottomrule
\end{tabular*}
\end{table}

\begin{table}[t]
\centering\footnotesize
\caption{Parameter count and FLOPs of the proposed method, and the Vanilla model. (The number of parameters and FLOPs are measured with \texttt{ptflops})}
\label{tab:footprint}
\begin{tabular*}{\linewidth}{@{\extracolsep{\fill}}@{}llcccc@{}}
\toprule
& & \multicolumn{2}{c}{Vanilla} & \multicolumn{2}{c}{Ours} \\
\cmidrule(lr){3-4}\cmidrule(lr){5-6}
Dataset & Backbone & \#param & FLOPs & \#param & FLOPs \\
\midrule
CIFAR100-LT & ResNet34   & 21.3M & 2.3G  & 21.6M  & 2.3G  \\ %0.56M 133.2M
ImageNet-LT & ResNeXt50  & 25.0M & 8.6G  & 29.2M  & 8.6G   \\
Places-LT   & ResNet152  & 60.2M & 23.2G & 64.4M  & 23.2G  \\
iNaturalist & ResNet50   & 40.2M & 8.2G  & 44.4M & 8.2G  \\
\bottomrule
\end{tabular*}
\end{table}

\subsection{Ablation Study}
\label{sec:ablation}

We conduct various ablation studies on the ETF-AM loss and the gating layer, as summarized in \cref{tab:ablation}.

\paragraph{Ablation of key components.} 
As shown in \cref{tab:ablation}, omitting ETF-AM leaves Many nearly unchanged (74.88\%) but drops Medium and Few by 2.4 and 4.9 points (overall 52.42\%).
Removing the gating layer reduces accuracy to 53.60\%, confirming that the masks aid inter-class separation.

These drops reflect that each component aligns with its role. Removing ETF-AM hurts the tail the most, consistent with keeping tail classes away from full collapse. Removing the gate yields a smaller, more uniform drop, consistent with the dimensionality of the axis.

\paragraph{Capacity-matched gating and margin baselines.} 
To isolate whether CASA's gain comes from added parameters or per-class node selection, we compare parameter-matched variants on CIFAR100-LT in~\cref{tab:ablation}.
CASA (54.73\%) outperforms the soft-attention gate (52.90\%) and the two-layer MLP head (48.80\%), demonstrating the effectiveness of the gating layer.
Replacing the adaptive margin with a fixed margin $\Gamma = 0.71$ drops accuracy to 52.80\%. The tail gain therefore comes from scaling the margin with class frequency, not from the margin alone.
Extended experimental analysis, including an ablation study on a larger dataset and qualitative analysis, is provided in \cref{sec:supp-ext} of the supplementary materials.

\begin{table}[t]
\setlength{\tabcolsep}{0.4mm} 
\centering
\caption{Component ablation on CIFAR100-LT: FDRs and Top-1 accuracy. \textbf{Bold} marks the best.}
\label{tab:ablation}
\scriptsize
% \resizebox{\columnwidth}{!}{%
\begin{tabular}{@{}l|cc|cccc@{}}
\toprule
Method & FDR (Train) & FDR (Test) & Many & Medium & Few & Avg. \\
\midrule
Full model & $9.94{\times}10^4$ & $1.28{\times}10^2$ & 74.71 & \textbf{55.77} & \textbf{31.65} & \textbf{54.73} \\ \midrule
 \rowcolor{lightgrey}\multicolumn{7}{l}{\textbf{Ablation on key components}}\\ \midrule

w/o ETF-AM                     & $3.36{\times}10^4$ & $1.23{\times}10^2$ & \textbf{74.88} & 53.34 & 26.74 & 52.42 \\
w/o Gating layer               & $9.46{\times}10^4$ & $1.24{\times}10^2$ & 74.20 & 54.10 & 30.60 & 53.60 \\  
\midrule
 \rowcolor{lightgrey}\multicolumn{7}{l}{\textbf{Capacity-matched gating and margin baselines}}\\ \midrule
 2-layer MLP & $3.84\times10^4$ & $1.05\times10^2$ & 70.30 & 45.00 & 29.50 & 48.80  \\ 
 Soft Attn. & $7.09\times10^4$ & $1.23\times10^2$ & 74.40 & 53.40 & 29.20 & 52.90  \\ 
Fixed margin ($\Gamma=0.71$) & $6.23{\times}10^4$ & $1.22{\times}10^2$ & 73.70 & 52.90 & 29.60 & 52.80  \\

\bottomrule
\end{tabular}%
% }
\end{table}

\begin{table}[t]
\centering\footnotesize
\setlength{\tabcolsep}{9pt}

\caption{Within-class spread $s_c$, averaged within each frequency split and normalized to the Many-class mean. In both datasets, the spread increases monotonically from head to tail, and ETF-AM maintains a larger tail spread than the CE baseline.}
\label{tab:spread}
  \begin{tabular}{@{}l|l|ccc@{}}
\toprule
 Dataset & Loss & Many & Medium & Few \\
\midrule
\multirow{2}{*}{CIFAR100-LT}   & ETF-AM     & 1.00 & 1.40 & 1.76 \\
                            & CE     & 1.00 & 1.18 & 1.33 \\ \midrule
\multirow{2}{*}{iNaturalist-2018}   & ETF-AM & 1.00 & 1.38 & 1.40 \\
                                    & CE & 1.00 & 1.04 & 1.01 \\
\bottomrule
\end{tabular}
\end{table}

\begin{figure*}[tb]
  \centering
  \includegraphics[width=0.75\linewidth]{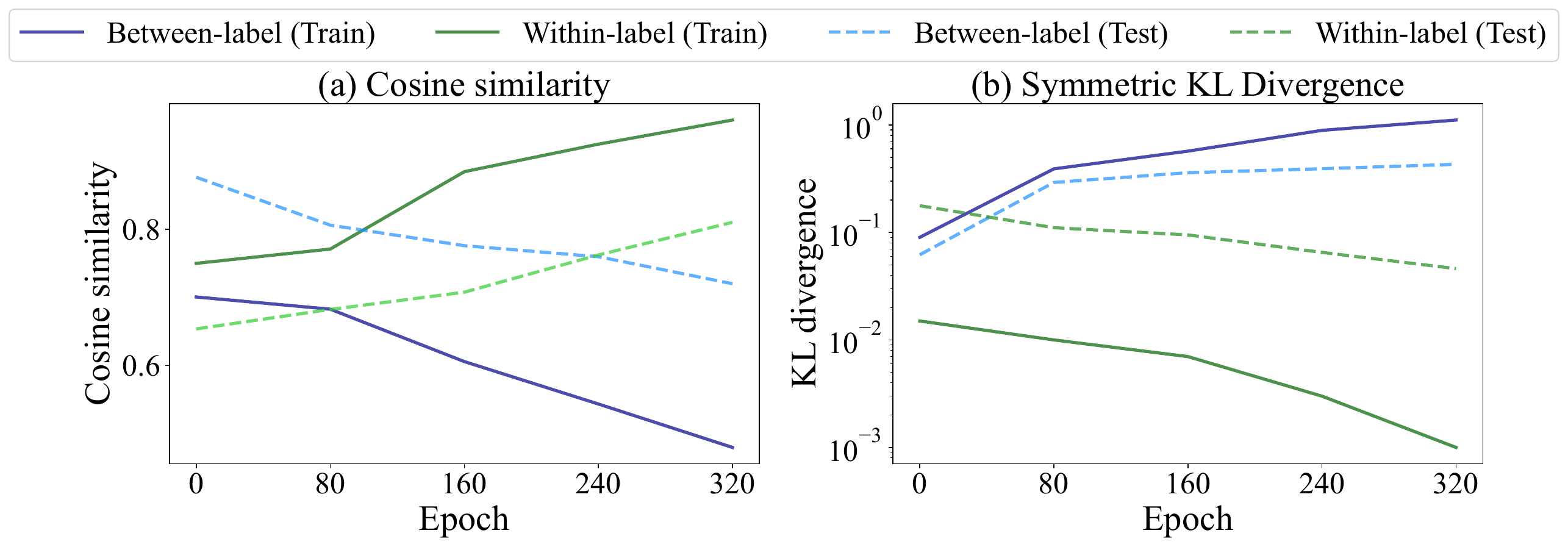}
  \caption{Epoch-wise cosine similarity (a) and symmetric KL divergence (b) of generated masks on CIFAR100-LT. \rev{Within-class masks grow more similar over training, whereas between-class masks separate.}}
  \label{fig:epoch-KL-Cos}
\end{figure*}

\begin{figure*}[tbh]
  \centering
  \includegraphics[width=0.45\linewidth]{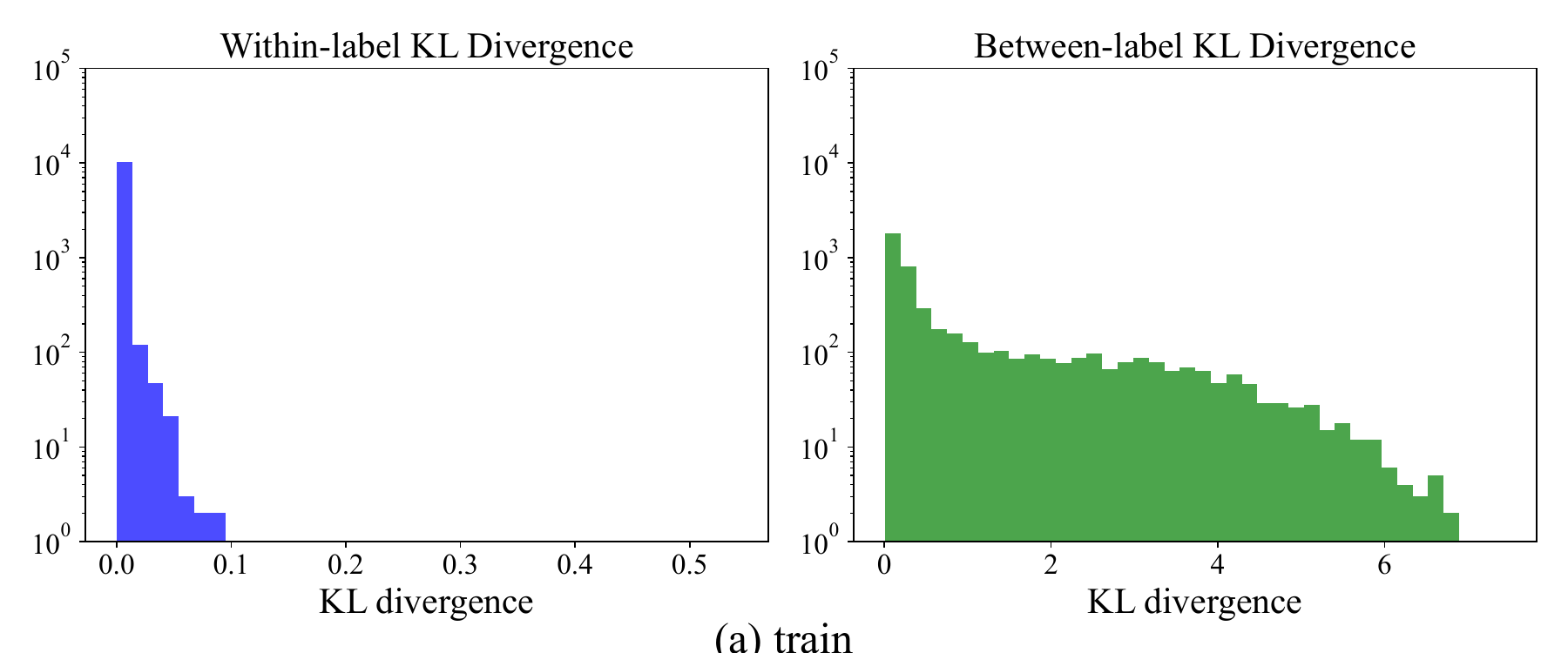}
  \includegraphics[width=0.45\linewidth]{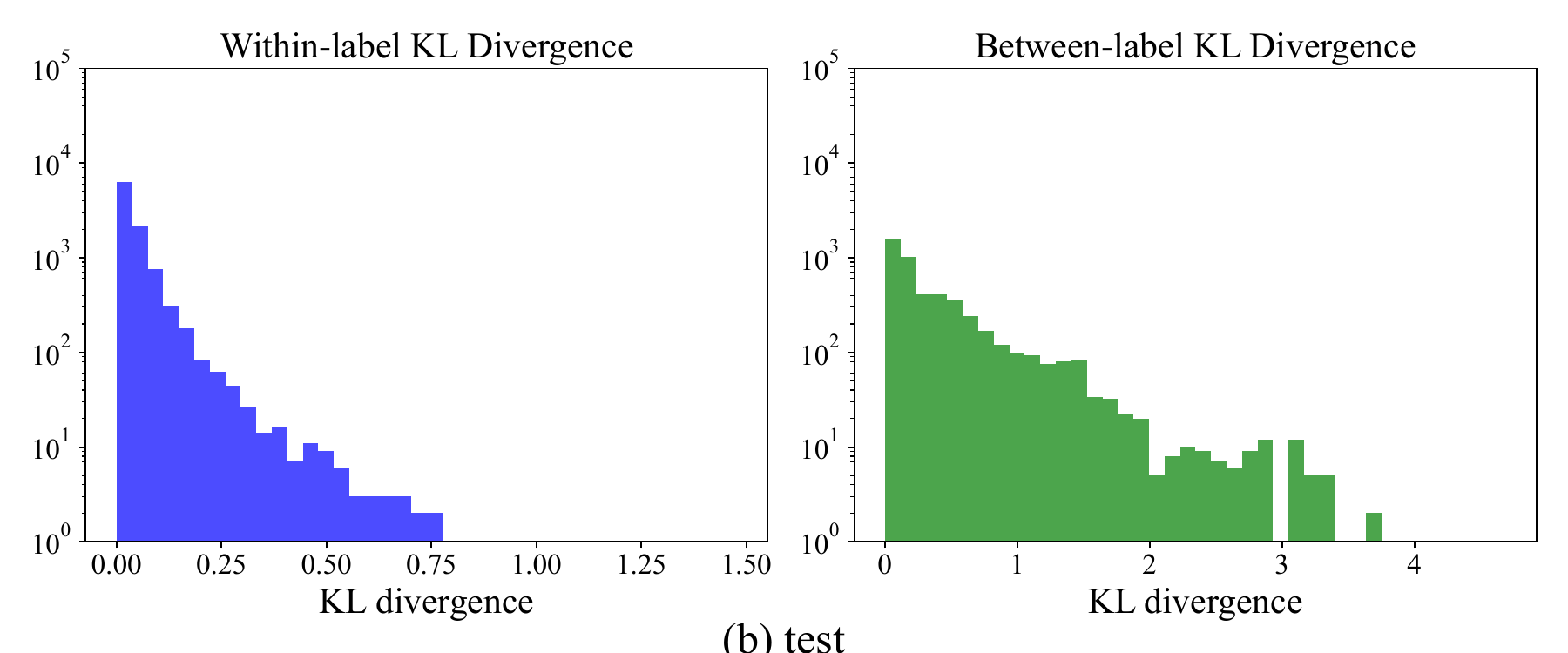}
  \caption{Histograms of symmetric KL divergence of generated mask (before thresholding on the training time) between sample pairs. (a) shows divergences computed from training set samples, and (b) displays divergences from test set samples. All histograms are plotted on a logarithmic scale.}
  \label{fig:supp-hist}
\end{figure*}

\subsection{Empirical Support for Partial Collapse}
\label{sec:partial-evidence}
\rev{\cref{thm:partial} shows that the risk is strictly worse at zero spread for classes below the threshold, and \cref{cor:scarcity} orders their optimal spread by $n_c$.}

% \paragraph{Tail classes retain larger spread.}
We measure the within-class spread $s_c$ (test set for CIFAR100-LT and validation set for iNaturalist-2018) of the trained models, as shown in~\cref{tab:spread}.
On CIFAR100-LT, the spread increases monotonically from Many to Few: the Few-class spread is $1.76\times$ the Many-class spread.
Compared with the cross-entropy baseline (without ETF-AM), ETF-AM retains a larger relative within-class spread on the Few classes (1.76 vs. 1.33), and this larger spread coexists with higher Few-class accuracy. We report relative spreads because the raw values are scale-dependent: on the Many-classes, ETF-AM and CE differ in $s_c$ ($18.4$ vs.\ $172.2$) as their feature norm ($8.6$ vs.\ $26.1$).
This is consistent with the partial-collapse view, in which tail classes benefit from within-class diversity rather than from full collapse.
The same monotonic pattern holds on iNaturalist-2018, where no simplex ETF exists, and the classifier is learned: the Few-class spread is $1.40\times$ the Many-class spread under ETF-AM, against $1.01\times$ for the CE baseline. 

This is the partial-collapse geometry: inter-class separation (NC2) is preserved while within-class collapse (NC1) is relaxed for tail classes.
\cref{thm:partial} acts only on the within-class spread and is agnostic to the
inter-class target. The retained tail spread on both datasets, including the
no-ETF iNaturalist regime, is consistent with it.
\rev{\Cref{sec:supp-ext} tests the zero-spread side with a controlled collapse experiment, and reports the frequency dependence of the spread.}

\subsection{Mask Consistency}
\label{sec:mask-consistency}

To support the dimensionality of collapse controlled by CASA, we evaluate the consistency of the mask generation process by analyzing the symmetric KL divergence and cosine similarity between mask distributions computed from sample pairs. \cref{fig:epoch-KL-Cos} illustrates the epoch-wise evolution of these metrics. The gating layer progressively decreases cosine similarity and increases KL divergence between samples of different classes, while showing the opposite trend within the same class. \rev{Same-class instances activate more similar subsets of nodes than
different-class instances.}

\begin{table}[tb]
\centering\footnotesize
\setlength{\tabcolsep}{15pt}
\caption{Cosine similarity and average KL divergence of generated masks between inter-class and intra-class on CIFAR100-LT.}
\label{tab:supp-mask}
  \begin{tabular}{@{}l  cc cc@{}}
\toprule
& \multicolumn{2}{c}{Cosine Similarity} & \multicolumn{2}{c}{Avg.\ KL Divergence} \\
\cmidrule(lr){2-3}\cmidrule(lr){4-5}
& Inter & Intra & Inter & Intra \\
\midrule
Train & 0.48 & 0.96 & 1.115 & 0.001 \\
Test  & 0.72 & 0.81 & 0.430 & 0.046 \\
\bottomrule
\end{tabular}
% \vspace{-5mm}
\end{table}

\cref{tab:supp-mask} quantifies the consistency of generated masks. On the train set, intra-class pairs reach cosine similarity 0.96 (KL 0.001) versus inter-class 0.48 (KL 1.115). On the test set, intra-class similarity drops to 0.81 (KL 0.046), while inter-class similarity 0.72 (KL 0.430) remains distinct. \rev{The gating layer therefore produces masks with class-related structure on both splits.} The KL-divergence histograms in \cref{fig:supp-hist} further validate this: on both the train and test sets, \rev{same-class pairs concentrate at lower divergence than inter-class pairs on both splits.} The intra-class divergence is slightly higher on the test set, reflecting the distribution shift between the long-tailed training set and the class-balanced test set rather than a loss of mask consistency.
Further discussion about masks and success/failure case analysis is in \cref{sec:supp-ext} of the supplementary material.

\section{Limitations}
\label{sec:limit}

Our method has some limitations.
\rev{(i) When $C>d+1$, a simplex ETF does not exist, so we utilize a learnable classifier head. We keep the same procedure, but $\lVert w_y-w_k\rVert$ is no longer constant across pairs: $\Gamma$ must be tuned, and the clearance holds per pair rather than per class.}
\rev{(ii) \cref{thm:partial} is established under a finite-sample class-conditional model (\cref{sec:supp-ufm}). The positivity of the optimal spread follows from the qualitative form of the assumption alone, while the sample-count dependence is grounded in prior analyses~\cite{galanti2022on,galanti2022generalization}. The cross-class ordering additionally requires $\phi$ to decrease throughout the feasible interval.}
We test the full-collapse-suboptimal side empirically and treat the interior optimum as a model-level prediction; a proof for the trained network remains open. 
(iii) We verify the claim for long-tailed image classification; transfer to other regimes, such as long-tailed detection or segmentation, remains open.

\section{Conclusion}
\label{sec:conclusion}

We reframed long-tailed recognition around controlled partial neural collapse: head classes effectively collapse along a fixed ETF while tail classes retain within-class diversity.
Under a finite-sample class-conditional model, we showed that full collapse fails to minimize risk for sample-scarce classes and that the optimal spread is ordered by class frequency; our measurements are consistent with that ordering.
Two coordinated controls on a backbone reach this regime: \rev{ETF-AM acts on the training-to-inference clearance}, and CASA on its dimensionality.
Our method achieves the best performance on CIFAR100-LT, ImageNet-LT, Places-LT, and iNaturalist-2018.

\section*{Acknowledgement}
This work was supported by the National Research Foundation of Korea (NRF) grant funded by the Korea government (MSIT) (RS-2025-00573160), the Institute of Information \& Communications Technology Planning \& Evaluation (IITP) - Information Technology Research Center (ITRC) grant funded by the Korea government (MSIT) (IITP-2026-RS-2023-00259703),
and the ``Advanced GPU Utilization Support Program'' funded by the Government of the Republic of Korea (Ministry of Science and ICT) (for H. Hwang and D. Kim); 
and by Institute of Information \& communications Technology Planning \& Evaluation (IITP) grant funded by the Korea government (MSIT) (No. RS-2025-02304183, Development of Optimization Code Conversion Technology for Heterogeneous AI Semiconductor-Based Large-Scale Models) (for C. Hur).

{
    \small
    \bibliographystyle{ieeenat_fullname}
    \bibliography{refs}
}

\clearpage

\clearpage
\appendix
\setcounter{page}{1}
\setcounter{theorem}{0}
\setcounter{assumption}{0}
\setcounter{corollary}{0}
\numberwithin{theorem}{section}
\numberwithin{assumption}{section}
\numberwithin{corollary}{section}
% \numberwithin{equation}{section}
% \numberwithin{table}{section}
% \numberwithin{figure}{section}
% \renewcommand{\thetable}{\Alph{table}}
% \renewcommand{\thefigure}{\Alph{figure}}
% \renewcommand{\thetheorem}{\Alph{theorem}}
% \renewcommand{\theassumption}{\Alph{assumption}}
% \renewcommand{\theequation}{\Alph{equation}}
% \renewcommand{\thetheorem}{\Alph{figure}}
\maketitlesupplementary

\section{Finite-Sample Model for the Optimality of Partial Collapse}
\label{sec:supp-ufm}

\rev{This section supports \cref{thm:partial,cor:scarcity} of the main paper.
We first show that the UFM, the standard abstraction for neural collapse, drives every class to \emph{full} collapse at its training optimum (\cref{sec:supp-ufm-train}).
The UFM contains no sampling process and no learned map from inputs to features, so it has no training-to-test gap and cannot express how the risk-optimal geometry depends on the class size.
We therefore analyze the finite-sample class-conditional model introduced in the main paper (\cref{sec:supp-ufm-model}) and prove that, for any class satisfying the sensitivity condition of \cref{eq:sensitivity}, the within-class spread that minimizes per-class risk is strictly positive (\cref{sec:supp-ufm-thm}).
The same section resolves that condition into a threshold on the training count and, under an additional monotonicity condition on $\phi$, orders the optimal spread across classes.}

\subsection{The Training Optimum Under the UFM Is Full Collapse}
\label{sec:supp-ufm-train}
We use the UFM, which treats last-layer features as free variables~\cite{dang2024neural,yang2022inducing}.
The features are $h_{c,i} = \mu_c + \varepsilon_{c,i}$ for class $c = 1, \dots, C$ and sample $i = 1, \dots, n_c$, with counts $n_1 \geq \cdots \geq n_C$, class mean $\mu_c$, zero-mean deviation $\varepsilon_{c,i}$, and within-class spread $s_c = \frac{1}{n_c}\sum_{i=1}^{n_c}\|\varepsilon_{c,i}\|^2$ as in the main paper.
The classifier provides a \emph{maximal-separation} inter-class geometry: a simplex ETF (with $\|w_c\|=1$ and $\langle w_c,w_{c'}\rangle=-\tfrac{1}{C-1}$ for $c\neq c'$) when $C\leq d+1$, which gives a positive separation margin $\gamma >0$ between distinct class means.
The proof uses this margin, not a specific frame.
A per-class budget with some finite radius $R$ couples the mean norm and the spread, $\|\mu_c\|^2 + s_c \leq R^2$.
With the margin cross-entropy of the main paper, the class-$c$ loss decreases as $\mu_c$ aligns with $w_c$ and as the true-class logit $w_c^\top \mu_c = \|\mu_c\|$ grows.
The budget then forces $s_c \to 0$: every unit of spread is spent at the cost of mean norm, and the loss prefers mean norm.
The adaptive margin $\varrho_c$ shifts the loss scale but not this conclusion.
Hence, the UFM training optimum is full collapse, $s_c = 0$, for every class and any imbalance profile.
The result is imbalance-agnostic because the inter-class target is held fixed and the argument acts only on the within-class spread $s_c$.
The imbalance reshapes the inter-class geometry studied elsewhere~\cite{fang2021exploring,thrampoulidis2022imbalance,dang2024neural}, but enters our analysis only through the test-risk dependence on $n_c$, not the training optimum.

\subsection{The Finite-Sample Class-Conditional Model}
\label{sec:supp-ufm-model}
The optimum above concerns the \emph{training} objective on free features; it says nothing about test performance, because the UFM has no data-generating distribution.
To reason about generalization, the main paper adopts a minimal population model for one class $c$ with training count $n_c$, in which the signed test margin of a random class-$c$ sample is
\begin{equation}
\label{eq:supp-margin-model}
\mathcal{M}_c(s_c) = (\gamma - a\,s_c) + Z_{\text{test}} + Z_{\text{est}}(s_c, n_c),
\end{equation}
\rev{with $a > 0$, irreducible test noise $Z_{\text{test}} \sim \mathcal{N}(0, \sigma^2)$ with $\sigma^2>0$, and an independent Gaussian finite-sample term $Z_{\text{est}}$.}
The sample is correctly classified when $\mathcal{M}_c(s_c) > 0$. The term $-a\,s_c$ is the separation cost: a larger within-class spread moves features closer to the boundary.
\rev{The term $Z_{\text{est}}$ is the finite-sample cost of the learned feature map, whose variance is the object of \cref{asm:reg} and \cref{eq:sensitivity}. Relaxing the spread reduces variance at rate $-V_s(0,n_c)$ while paying separation at rate $a$, and the trade favors a positive spread when the former dominates.
The modeling content of \cref{asm:reg} lies in the dependence of $V$ on the training count $n_c$. The finite-sample cost of estimating a class-$c$ feature map grows as $n_c$ shrinks.
This dependence is what the UFM lacks and what makes the analysis sensitive to imbalance.}

\rev{\citet{galanti2022on} establish within-class variability as a key determinant of finite-sample error, with estimates that degrade as the number of available samples shrinks. \citet{hui2022limitations} observe that train-time variability collapse and test accuracy can be anti-correlated, indicating that driving the spread to zero degrades the learned feature map rather than improving it; this is what $\phi'(0)<0$ encodes.}

Two simplifications are built into \cref{eq:supp-margin-model}.
First, the model identifies the test-time spread with the training spread: the same $s_c$ enters both the separation cost and the estimation variance, which assumes the learned feature map reproduces its training spread at test time.
\rev{Second, classes are decoupled, and $a$, $\gamma$, $\sigma$, and $\phi$ are shared across classes; the class count $n_c$ is the varying model parameter in the threshold result. The separation cost $-a\,s_c$ therefore prices only a class's own interference toward the boundary, not that of its neighbors.}
Both simplifications keep the risk separable by class, which makes the closed-form analysis below possible.
\rev{\subsection{Proofs of \texorpdfstring{\cref{thm:partial} and \cref{cor:scarcity}}{Theorem 1 and Corollary 1}}
\label{sec:supp-ufm-thm}
The balanced test risk weights each class equally, so it suffices to minimize the per-class error $\mathcal{R}_c(s_c) = \Pr[\mathcal{M}_c(s_c) \leq 0]$.
We prove \cref{thm:partial} of the main paper, then \cref{cor:scarcity}, which resolves the sensitivity condition into a threshold on the training count.
}
\rev{\begin{theorem}[\cref{thm:partial}, suboptimality of full collapse]
\label{thm:partial-supple}
Let $V$ be as in \cref{asm:reg}. For any class $c$ satisfying \cref{eq:sensitivity}, the per-class risk on $[0,\gamma/a]$ attains its global minimum in the interior. Writing $s_c^\star$ for the smallest global minimizer,
\begin{equation}
s_c^\star>0 \quad\text{and}\quad \mathcal{R}_c(0)>\mathcal{R}_c(s_c^\star).
\end{equation}
\end{theorem}
}
\rev{\begin{proof}[Proof of \cref{thm:partial}]
Since $Z_{\text{test}}$ and $Z_{\text{est}}$ are independent zero-mean Gaussians, $\mathcal{M}_c(s)$ is Gaussian with mean $\gamma-as$ and variance $D(s)=\sigma^{2}+V(s,n_c)$.
With $\Phi$ the standard Gaussian CDF,
\begin{equation}
\label{eq:supp-risk}
\mathcal{R}_c(s)=\Phi\!\left(-g(s)\right),\qquad g(s)=\frac{\gamma-as}{\sqrt{D(s)}} ,
\end{equation}
and since $\Phi$ is strictly increasing, minimizing $\mathcal{R}_c$ on $[0,\gamma/a]$ is equivalent to maximizing $g$.

By $\sigma^2>0$, \cref{asm:reg} gives $0<D(s)=\sigma^{2}+V(s,n_c)<\infty$ throughout the compact interval $[0,\gamma/a]$. Hence $g$ is continuous on that interval and differentiable in its interior. Differentiating,
\begin{equation}
\label{eq:supp-gprime}
g'(s)=-\,\frac{2a\,D(s)+(\gamma-as)\,V_s(s,n_c)}{2\,D(s)^{3/2}} .
\end{equation}
The denominator is positive, so at $s=0$,
\begin{equation}
g'(0)>0 \iff -V_s(0,n_c)>\frac{2a\left[\sigma^{2}+V(0,n_c)\right]}{\gamma},
\end{equation}
which is exactly \cref{eq:sensitivity}. Let class $c$ satisfy it.

At the two endpoints, $g(0)=\gamma/\sqrt{D(0)}>0$, while $g(\gamma/a)=0$.
Since $g'(0)>0$, there is $\delta>0$ with $g(\delta)>g(0)$.
Since $g$ is continuous on the compact interval, it attains a maximum, and that value exceeds both endpoints. The maximum therefore lies in $(0, \gamma/a)$.
Hence $g$ has an interior global maximizer, equivalently $\mathcal{R}_c$ has an interior global minimizer.

Let $\mathcal{S}(n_c)=\operatorname*{arg\,max}_{s\in[0,\gamma/a]}g(s)$. The preceding argument shows that $\mathcal{S}(n_c)$ is nonempty and contained in $(0,\gamma/a)$. It is a closed subset of a compact interval, so it has a smallest element; define $s_c^\star$ to be that element. Since $\max_s g(s)>g(0)$,
\begin{equation}
\begin{aligned}
s_c^{\star}&>0 \quad\text{and}\\\quad \mathcal{R}_c(s_c^{\star})=\Phi\!\left(-g(s_c^{\star})\right)&<\Phi\!\left(-g(0)\right)=\mathcal{R}_c(0).
\end{aligned}
\end{equation}
Full collapse is therefore strictly suboptimal for any class satisfying \cref{eq:sensitivity}.
The proof of \cref{thm:partial} is finished.
\end{proof}
}
\rev{
\cref{thm:partial} is a per-class existence statement. It applies to whichever classes satisfy \cref{eq:sensitivity} and is silent about the rest. The assumed $1/n_c$ scaling turns this condition into an explicit threshold on the training count.

\begin{corollary}[\cref{cor:scarcity}, scarcity-dependent spread]
\label{cor:scarcity-supple}
    Let $n^\star$ be the sample count at which \cref{eq:sensitivity} changes sign. Then \cref{eq:sensitivity} holds if and only if $0<n_c<n^\star$. If, in addition, $\phi'(s)<0$ for every $s\in(0,\gamma/a)$, then for any two classes $c,c'$ with $n_{c'}<n_c<n^\star$,
    \begin{equation}
    s_c^{\star}<s_{c'}^{\star}.
    \end{equation}
\end{corollary}
}
\rev{\begin{proof}[Proof of \cref{cor:scarcity}]
\emph{Threshold.}
% \emph{Threshold.}
By \cref{asm:reg}, $V(0,n_c)=\phi(0)/n_c$ and $V_s(0,n_c)=\phi'(0)/n_c$,
so \cref{eq:sensitivity} reads
\begin{equation}
-\frac{\phi'(0)}{n_c}>\frac{2a}{\gamma}\left[\sigma^{2}+\frac{\phi(0)}{n_c}\right].
\end{equation}
Multiplying by $\gamma n_c>0$ and rearranging gives
$2a\sigma^{2}n_c<-\gamma\phi'(0)-2a\phi(0)$, i.e.\ $0<n_c<n^\star$ with
\begin{equation}
n^{\star}=\frac{-\gamma\,\phi'(0)-2a\,\phi(0)}{2a\sigma^{2}} .
\end{equation}

\emph{Ordering.} Assume in addition $\phi'(s)<0$ on $(0,\gamma/a)$.
Since $g\ge 0$, maximizing $g$ is equivalent to maximizing
\begin{equation}
F(s,n)=g(s)^{2}=\frac{n(\gamma-as)^{2}}{n\sigma^{2}+\phi(s)} .
\end{equation}
For $n<n^\star$, \cref{thm:partial} places every maximizer in $(0,\gamma/a)$, where $F>0$; on this open interval set $\Lambda(s,n)=\log F(s,n)$. For $s>s'$,
\begin{equation}
    \begin{aligned}
        \frac{\partial}{\partial n}\big[\Lambda(s,n)-\Lambda(s',n)\big] &=\\
        \sigma^{2}\!\left[\frac{1}{n\sigma^{2}+\phi(s')}-\frac{1}{n\sigma^{2}+\phi(s)}\right]&<0 ,
    \end{aligned}
\end{equation}
since $\phi(s)<\phi(s')$ and $\sigma^{2}>0$. Thus $\Lambda$ has strictly decreasing differences in $(s,n)$.

Let $n_{c'}<n_c<n^\star$ and take any $s\in\mathcal{S}(n_c)$ and $s'\in\mathcal{S}(n_{c'})$. If $s>s'$, optimality gives $\Lambda(s,n_c)-\Lambda(s',n_c)\ge 0$ and $\Lambda(s,n_{c'})-\Lambda(s',n_{c'})\le 0$, contradicting strictly decreasing differences. Hence $s\le s'$ for every such pair.

Both maximizers are interior, so each satisfies $-2a/(\gamma-as)=\phi'(s)/(n\sigma^{2}+\phi(s))$. The left side does not depend on $n$ while the right side is strictly monotone in $n$ whenever $\phi'(s)<0$, so no point is optimal for both counts. Therefore $s_c^{\star}<s_{c'}^{\star}$.
The proof of \cref{cor:scarcity} is finished.
\end{proof}
}

\rev{\paragraph{Regarding the condition.}
The threshold and ordering parts have different roles. The threshold identifies the classes for which full collapse is suboptimal. The ordering is the part the design uses, since ETF-AM follows the same monotone ranking in $n_c$. It is also the part we measure (\cref{sec:supp-slope}), while $n^\star$ depends on constants the model leaves unspecified. The threshold uses the sign of $\phi'$ at zero, whereas the ordering requires it on the whole interval $(0,\gamma/a)$.
}

\rev{\paragraph{Scope of the inter-class target.}
The inter-class target is the simplex ETF of \cref{sec:method} when $C\le d+1$. For $C>d+1$ (e.g., iNaturalist-2018 with $C=8{,}142>d=2{,}048$) no simplex ETF exists, and we learn the classifier instead. The proof concerns the within-class spread $s_c$ and holds unchanged for either inter-class target, since it uses only the base margin $\gamma>0$.
}

\rev{\section{Reference Scale Determination Under ETF}
\label{sec:supp-margin}}

In this section, we elaborate on the derivation of the reference scale under the ETF constraint as mentioned in \cref{sec:etfam-margin} of the main paper.

Under the ETF constraint, all class weight vectors $w_c$ lie on the unit hypersphere, and their pairwise inner products satisfy:
\begin{equation}
\langle w_c, w_{c'} \rangle = -\frac{1}{C - 1}, \quad \text{for } c \neq c',
\end{equation}
where $C$ denotes the total number of classes.

Based on this property, the squared Euclidean distance between any two unit-norm vectors can be derived as follows:
\begin{equation}
\|w_c - w_{c'}\|^2 = 1 + 1 - 2\Big({-\frac{1}{C - 1}}\Big) = 2 + \frac{2}{C - 1}.
\end{equation}

Thus, the Euclidean distance between any two distinct class weight vectors is given by:
\begin{equation}
\|w_c - w_{c'}\| = \sqrt{2 + \frac{2}{C - 1}}.
\end{equation}

Assuming that the decision boundary between two class centers is placed at the midpoint, the effective margin (the distance from a class center to the boundary) is:
\begin{equation}
\text{Margin} = \tfrac{1}{2}\sqrt{2 + \tfrac{2}{C - 1}}.
\end{equation}

For a sufficiently large $C$, the term $\tfrac{2}{C-1}$ becomes negligible, and the margin converges to:
\begin{equation}
\text{Margin} = \tfrac{1}{2}\sqrt{2 + \tfrac{2}{C - 1}} \to \tfrac{1}{2}\sqrt{2} \approx 0.7071.
\end{equation}

This derivation shows that $\Gamma$ asymptotically approaches $\tfrac{\sqrt{2}}{2}$, ensuring that the margin adheres to the intrinsic geometric constraints imposed by the ETF configuration.

\begin{table*}[tb]
\centering\small
\caption{\rev{Hyperparameters used in experiments. $\Gamma$ is a fixed value, except for iNaturalist2018, which cannot have an ETF structure}.}
\label{tab:supp-hp}
% \resizebox{\columnwidth}{!}{%
\begin{tabular}{@{}l|cccc@{}}
\toprule
& CIFAR100-LT & ImageNet-LT & Places-LT & iNaturalist2018\\
\midrule
\multicolumn{4}{l}{\emph{Data Settings}} \\ \midrule
\# Classes        & 100 & 1000 & 365 & 8142 \\
Batch Size       & 64  & 128  & 128 & 512 \\
\midrule
\multicolumn{4}{l}{\emph{Optimizer Parameters}} \\ \midrule
Total Epochs     & 320 & 200  & 30 & 200 \\
Optimizer        & SGD & SGD  & SGD & SGD \\
Momentum         & 0.9 & 0.9  & 0.9 & 0.9 \\
Scheduler        & Cos & MultiStep & ReduceLROnPlateau & Cos \\
Base LR          & 0.01 & 0.05 & 0.3 & 0.05 \\
Weight Decay     & 0.005 & 0.00024 & 0.00024 & 0.0005 \\
\rev{$\beta$ (Feature Reg.)}     & \rev{0.01} & \rev{0.0001}& \rev{0.00003} & \rev{0.0001}\\
\rev{LA Coef. }          & \rev{0.65} & \rev{0.2} & \rev{0.35} & \rev{0.15} \\
\rev{$\Gamma$ (Reference Scale)} & \rev{0.71} & \rev{0.71} & \rev{0.71} & \rev{0.9} \\
\midrule
\multicolumn{4}{l}{\emph{Loss Hyperparameters}} \\ \midrule
$\vartheta$ (spike) & 0.5 & 0.5 & 0.3 & 0.5\\
$\varpi$ (slab)    & 0.1 & 0.1 & 0.3 & 0.1\\
\bottomrule
\end{tabular}%
% }
\end{table*}

\rev{
\section{Clearance Derivation Under the ETF Constraint}
\label{sec:supp-clearance}

This section derives the training-to-inference clearance induced by the adaptive margin $\varrho_y$, as used in \cref{sec:etfam} of the main paper.

\paragraph{Two parallel hyperplanes.}
With the logit $z_c=w_c^\top u+b_c$ and the masked feature $u=M\odot h$, the inference rule for a true class $y$ against a class $k$ is $z_y-z_k>0$ as:
\begin{equation}
\label{eq:supp-infer-plane}
(w_y-w_k)^\top u+(b_y-b_k)>0 ,
\end{equation}
whereas subtracting $\varrho_y$ from the true-class logit makes the training condition $z_y-\varrho_y>z_k$, i.e.
\begin{equation}
\label{eq:supp-train-plane}
(w_y-w_k)^\top u+(b_y-b_k)-\varrho_y>0 .
\end{equation}
The two hyperplanes share the normal vector $w_y-w_k$ and are therefore parallel; they differ only in their offsets. The Euclidean distance between them is
\begin{equation}
\label{eq:supp-clearance-pair}
\Delta_{yk}=\frac{\varrho_y}{\lVert w_y-w_k\rVert} .
\end{equation}
We call $\Delta_{yk}$ the \emph{clearance}: a feature meeting \cref{eq:supp-train-plane} sits at least $\Delta_{yk}$ from the inference boundary of \cref{eq:supp-infer-plane} once the margin is removed.

\paragraph{Independence from the competitor.}
Under the ETF constraint, $\lVert w_y-w_k\rVert$ is the same for every pair (\cref{eq:dist}), and by \cref{eq:maxmargin} it equals $2\Gamma$. Hence \begin{equation}
\label{eq:supp-clearance-y}
\Delta_{yk}=\Delta_y=\frac{\varrho_y}{2\Gamma} ,
\end{equation}
which depends solely on the true class. On a learnable classifier head, $\lVert w_y-w_k\rVert$ varies by pair and over training, so the same logit margin $\varrho_y$ maps to a clearance that is neither constant across competitors nor stable during optimization.

\paragraph{Clearance under gating.}
Both hyperplanes in \cref{eq:supp-infer-plane,eq:supp-train-plane} are defined in $u$-space, and the mask enters only through $u=M\odot h$. The normal is $w_y-w_k$, whose length the ETF fixes at $2\Gamma$ irrespective of $M$, so the clearance of \cref{eq:supp-clearance-y} is unchanged by gating: the mask selects which coordinates carry the feature, not the classifier geometry from which $\Delta_y$ is derived.

\paragraph{Scope of the clearance.}
$\Delta_y$ is a nominal training-to-inference buffer. It bounds the boundary-normal variation a feature can absorb between the two rules, but it does not set the learned cloud radius or guarantee a particular within-class spread.}

\begin{table*}[tbh]

\centering \small
\caption{Extended ablation study on CIFAR100-LT: FDR on train and test sets, and Top-1 Accuracy for Many, Medium, Few, and overall Average classes.}
\setlength{\tabcolsep}{8pt}

\begin{tabular}{@{}l|cc|cccc@{}}
\toprule
Method & FDR (Tr) & FDR (Te) & Many & Med & Few & Avg. \\
 \midrule
$ \vartheta{=}0.5,\varpi{=}0.1$
%,s {=} 1.0$, \Gamma {=} 0.71$
& $9.94\times10^4$           & ${1.28\times10^2}$ & 74.71 & \textbf{55.77} & \textbf{31.65}    & \textbf{{54.73}} \\ \midrule
% w/o ETF                               & $3.00\times10^4$         & $7.04\times10^1$          & 36.82           & 34.31              & \textbf{38.32} & 36.41              \\
w/o ETF-AM                                       & $3.36\times10^4$         & $1.23\times10^2$          & \textbf{74.88}  & 53.34              & $26.74$          & 52.42              \\
w/o Gating                                 & $9.46\times10^4$         & $1.24\times10^2$          & 74.20           & 54.10              & 30.60              & 53.60              \\ \midrule
$\vartheta{=}0.9,\varpi{=}0.1$                         & $9.65\times10^4$           & $1.24\times10^2$          & 73.00           & 55.48              & 31.29              & 53.56              \\
$\vartheta{=}0.7,\varpi{=}0.1$                         & $9.76\times10^4$           & $1.24\times10^2$          & 73.52           & 55.26              & 31.48              & 54.10              \\
$\vartheta{=}0.3,\varpi{=}0.1$                         & $9.80\times10^4$           & $1.25\times10^2$          & 73.85           & 54.22              & 31.73              & 54.22              \\
$\vartheta{=}0.1,\varpi{=}0.1$                         & {$9.84\times10^4$}           & $1.25\times10^2$          & 74.08           & 54.59              & {32.12}              & 54.30              \\ 
$\vartheta{=}0.5,\varpi{=}0.9$                         & $9.29\times10^4$           & $1.24\times10^2$          & 71.43           & 52.40              & 28.97              & 51.22              \\
$\vartheta{=}0.5,\varpi{=}0.7$                         & $9.32\times10^4$           & $1.24\times10^2$          & 72.61           & 53.40              & 29.87              & 52.19              \\
$\vartheta{=}0.5,\varpi{=}0.5$                         & $9.48\times10^4$           & $1.25\times10^2$          & 73.32           & 54.21              & 30.82              & 53.24              \\
$\vartheta{=}0.5,\varpi{=}0.3$                         & $9.47\times10^4$           & {$1.27\times10^2$}          & 73.55           &{55.60}             & 31.90              & 53.72              \\ \bottomrule
\end{tabular}

% \vspace{-6pt}

\label{tab:supp-cifar100_abl}
\end{table*}

\begin{table*}[tbh]
\centering\small
\caption{Extended Ablation study on ImageNet-LT with ResNeXt50: FDR on train and test sets, and Top-1 Accuracy for Many, Medium, Few, and overall Average classes. \textbf{Bold} marks the best.}
\label{tab:supp-imagenet-abl}
\setlength{\tabcolsep}{8pt}
\begin{tabular}{@{}l|cc|cccc@{}}
\toprule
Method & FDR (Tr) & FDR (Te) & Many & Med & Few & Avg. \\
\midrule
$\vartheta{=}0.5,\varpi{=}0.1$ & $9.19{\times}10^2$ & $7.25{\times}10^2$ & \textbf{65.20} & \textbf{55.35} & 38.87 & \textbf{56.90} \\ \midrule
% w/o ETF                        & $4.45{\times}10^2$ & $1.27{\times}10^2$ & 57.62 & 28.15 & 3.11  & 36.09 \\
w/o ETF-AM                     & $5.20{\times}10^2$ & $6.16{\times}10^2$ & 65.17 & 52.19 & 27.52 & 53.83 \\
w/o Gating                     & $8.89{\times}10^2$ & $7.02{\times}10^2$ & 64.54 & 52.87 & 30.73 & 54.34 \\ \midrule
$\vartheta{=}0.9,\varpi{=}0.1$ & $8.42{\times}10^2$ & $6.11{\times}10^2$ & 64.05 & 54.81 & 36.76 & 55.91 \\
$\vartheta{=}0.7,\varpi{=}0.1$ & $8.54{\times}10^2$ & $6.42{\times}10^2$ & 62.25 & 53.63 & 34.33 & 54.32 \\
$\vartheta{=}0.3,\varpi{=}0.1$ & $8.64{\times}10^2$ & $6.62{\times}10^2$ & 63.24 & 54.73 & 32.91 & 55.03 \\
$\vartheta{=}0.1,\varpi{=}0.1$ & $8.97{\times}10^2$ & $6.85{\times}10^2$ & 64.98 & 55.12 & \textbf{38.71} & 56.68 \\
$\vartheta{=}0.5,\varpi{=}0.9$ & $8.21{\times}10^2$ & $6.53{\times}10^2$ & 60.33 & 53.29 & 32.35 & 53.15 \\
$\vartheta{=}0.5,\varpi{=}0.7$ & $8.35{\times}10^2$ & $6.62{\times}10^2$ & 62.48 & 53.78 & 33.01 & 54.30 \\
$\vartheta{=}0.5,\varpi{=}0.5$ & $8.50{\times}10^2$ & $6.75{\times}10^2$ & 63.27 & 54.18 & 35.27 & 55.10 \\
$\vartheta{=}0.5,\varpi{=}0.3$ & $9.10{\times}10^2$ & $7.18{\times}10^2$ & 64.52 & 55.20 & 38.96 & 56.52 \\
\bottomrule
\end{tabular}
\end{table*}

\begin{figure*}[t]
  \centering
  \includegraphics[width=0.7\linewidth]{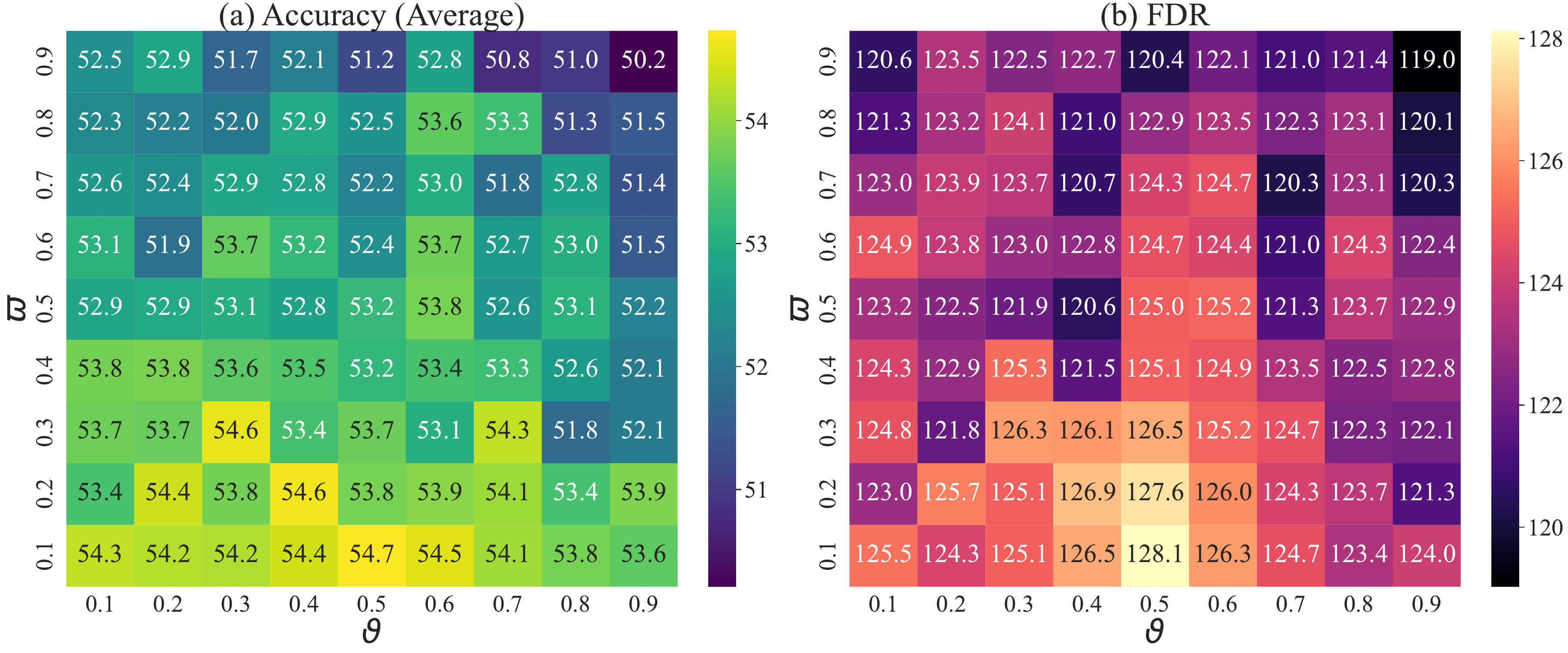}
  \caption{Hyperparameter ($\vartheta$ and $\varpi$) heatmap on CIFAR100-LT. (a) Average accuracy on the test set, (b) FDR score on the test set.}
  \label{fig:supp-heat}
\end{figure*}

\section{Detailed Formulation of Evaluation Metrics}
\label{sec:supp-metrics}

\subsection{Fisher Discriminant Ratio}
Given $N$ samples $x_i \in \mathbb{R}^D$ and $C$ classes, where $n_c$ is the number of samples in class $c$ and $\mu_c$ is the mean feature of class $c$, we define the following:
\begin{align}
S_B &= \sum_{c=1}^{C} n_c (\mu_c - \mu)(\mu_c - \mu)^\top,\\
S_W &= \sum_{c=1}^{C} \sum_{i \in I_c} (x_i - \mu_c)(x_i - \mu_c)^\top,
\end{align}
where $\mu$ is the overall mean of all samples.

FDR is defined as the ratio of the trace of the between-class scatter to the trace of the within-class scatter:
\begin{equation}
\text{FDR} = \frac{\text{tr}(S_B)}{\text{tr}(S_W)}.
\end{equation}

A high FDR indicates that the features are more compact within each class (lower $S_W$) and more separated between classes (higher $S_B$), which is a characteristic of effective NC. We use FDR to verify that inter-class separation is preserved, not to argue for uniform within-class compactness. Because tail classes are sample-scarce, the tail spread we deliberately retain contributes little to the overall tr($S_W$).

\subsection{Metrics for Mask Consistency}
Given two mask distribution vectors $m_i$ and $m_j$, the cosine similarity is computed as:
\begin{equation}
\text{CosSim}(m_i, m_j) = \frac{m_i \cdot m_j}{\|m_i\|\,\|m_j\|}.
\end{equation}

The symmetric KL divergence is defined as the sum of the Kullback-Leibler divergences \cite{kullback1951information} in both directions:
\begin{equation}
\text{KL}_{\text{sym}}(m_i, m_j) = \text{KL}(m_i \| m_j) + \text{KL}(m_j \| m_i),
\end{equation}
where, for two probability distributions $P$ and $Q$,
\begin{equation}
\text{KL}(P \| Q) = \sum_{k} P(k) \log \frac{P(k)}{Q(k)}.
\end{equation}

These metrics are computed for pairs of samples within the same class (intra-class) and across different classes (inter-class) to evaluate the degree of similarity in the generated mask distributions.

\section{Training Settings}
\label{sec:supp-train}

For our experimental environment, we use an Ubuntu 22.04 server with an Intel i9-10900X CPU, 64GB of RAM, and an NVIDIA 5090 GPU with CUDA 13.0, and an Ubuntu 22.04 server with two AMD EPYC 9654 CPUs, 512GB of RAM, and four NVIDIA PRO6000 Workstation Edition GPUs with CUDA 13.0. All experiments are implemented using PyTorch 2.10.0 with Python 3.13.

Detailed hyperparameters are provided in \cref{tab:supp-hp}. For Places-LT ($\rho=996$), we set $(\vartheta,\varpi)=(0.3,0.3)$. This setting activates more weights for the Few classes, thereby enhancing their feature representation and overall performance under extreme imbalance.

\section{Extended Experiments and Analysis}
\label{sec:supp-ext}

\rev{\subsection{Frequency Dependence of the Within-Class Spread}
\label{sec:supp-slope}

\Cref{cor:scarcity} gives a model-level threshold and, under the additional monotonicity condition, an ordering of the risk-optimal spreads by class frequency. The ordering predicts that the learned spread decreases as $n_c$ increases. We test this descriptive implication using the backbone features. For a split $\mathcal{S}$, let $N_c^{\mathcal{S}}$ denote the number of its class-$c$ samples. We compute
\begin{equation}
    \bar h_c^{\mathcal{S}}
    =\frac{1}{N_c^{\mathcal{S}}}\sum_{i\in\mathcal{S}:y_i=c}h_i,
    \label{eq:supp-empirical-feat}
\end{equation}
\begin{equation}
    s_c^{\mathcal{S}}
    =\frac{1}{N_c^{\mathcal{S}}}\sum_{i\in\mathcal{S}:y_i=c}
    \|h_i-\bar h_c^{\mathcal{S}}\|_2^2.
    \label{eq:supp-empirical-spread}
\end{equation}
For both the training and Test/Val. set, the predictor is the training count $n_c$, not the number of samples in the evaluation split. Specifically, with $x_c=\log n_c$ and $y_c^{\mathcal{S}}=\log s_c^{\mathcal{S}}$, the reported slope is the unweighted Ordinary Least Squares (OLS) coefficient
\begin{equation}
\widehat{\beta}_{\mathcal{S}}
=\frac{\sum_{c\in\mathcal{C}_{\mathcal{S}}}(x_c-\bar x_{\mathcal{S}})
(y_c^{\mathcal{S}}-\bar y_{\mathcal{S}})}
{\sum_{c\in\mathcal{C}_{\mathcal{S}}}(x_c-\bar x_{\mathcal{S}})^2},
\label{eq:supp-ols-slope}
\end{equation}
where $\bar x_{\mathcal{S}}$ and $\bar y_{\mathcal{S}}$ are the corresponding means over the set of classes $\mathcal{C}_{\mathcal{S}}$. Thus, each class contributes one equally weighted observation. We use natural logarithms; the fit includes 100 classes on CIFAR100-LT and 8,142 classes on iNaturalist. As \cref{tab:supp-slope} shows, all four slopes are negative, and every split retains Many\,$<$\, Medium\,$<$\, Few.

\begin{table}[tb]
\centering\footnotesize
\caption{\rev{Log--log OLS slope of class-wise within-class spread $s_c$ against the training count $n_c$. All slopes are negative on both splits.}}
\label{tab:supp-slope}
\begin{tabular}{@{}l|cc@{}}
\toprule
Dataset & Train & Test / Val. \\
\midrule
CIFAR100-LT      & $-0.602$ & $-0.249$ \\
iNaturalist-2018 & $-0.120$ & $-0.124$ \\
\bottomrule
\end{tabular}
\end{table}

We compare signs and orderings rather than magnitudes. The training and test splits are differently distributed. The model also identifies the training and test spread (\cref{sec:supp-ufm-model}), a simplification we adopt for tractability. The iNaturalist slopes agree closely across splits, though they are smaller in magnitude than on CIFAR.}

\rev{\subsection{Controlled Collapse Intervention}
\label{sec:supp-lambda}
To examine the effect of stronger within-class compression, we add a penalty that contracts Few-class backbone features toward their running class means,

\begin{equation}
\mathcal{L}_{\mathrm{col}}
=\frac{1}{|\mathcal{B}|}\sum_{i\in\mathcal{B}}
\mathbb{I}[\,n_{y_i}\le 20\,]\,
\|h_i-\widetilde{\mu}_{y_i}\|_2^{2},
\label{eq:supp-collapse-penalty}
\end{equation}
and train with $\mathcal{L}+\lambda\mathcal{L}_{\mathrm{col}}$, varying $\lambda$. Here $\mathcal{B}$ is the mini-batch, and $\widetilde{\mu}_c$ is a running mean of class $c$ with momentum $0.9$ initialized at the first batch mean containing the class. The penalty acts on the raw backbone feature $h$.

For every seed, each $\lambda>0$ setting yields substantially lower normalized Few-class spread on the clean training set than $\lambda=0$, although the training-set spread is not monotonic across the three nonzero values. On the test set, both normalized spread and post-LA Few-class accuracy decrease monotonically with $\lambda$ for all three seeds. From $\lambda=0$ to $0.3$, normalized Few-class spread decreases by $64.5{\pm}17.8\%$ on the clean training set and $19.7{\pm}1.1\%$ on the test set. Over the same change, post-LA Few-class accuracy decreases by $3.69{\pm}1.35$\, points, while training balanced accuracy remains at least $99.6\%$. We therefore view this sweep as a one-sided stress test: increasing the contraction penalty from the $\lambda=0$ baseline reduces test spread but degrades post-LA Few-class performance. The $\lambda=0.3$ setting does not enforce exact $s_c=0$, and the sweep does not identify an optimum or show that spread alone causes the accuracy change.

\begin{table}[tb]
\centering\footnotesize
\caption{\rev{Three-seed sweep of the Few-class contraction penalty with ResNet-34 on CIFAR100-LT. Each entry reports the mean$\pm$sample SD.}}
\label{tab:collapse_penalty}
\begin{tabular}{@{}cccc@{}}
\toprule
$\lambda$ & \shortstack{Train $s_c$ ($\times10^{-3}$)}
& \shortstack{Test $s_c$}
& \shortstack{Few Acc. (\%)} \\
\midrule
$0.0$ & $4.045{\pm}0.33$ & $0.3325{\pm}0.0033$ & $33.34{\pm}1.29$ \\
$0.1$ & $0.991{\pm}0.04$ & $0.2925{\pm}0.0048$ & $32.15{\pm}0.76$ \\
$0.2$ & $0.855{\pm}0.03$ & $0.2748{\pm}0.0045$ & $30.62{\pm}1.45$ \\
$0.3$ & $1.440{\pm}0.74$ & $0.2670{\pm}0.0014$ & $29.66{\pm}1.68$ \\
\bottomrule
\end{tabular}
\end{table}}

\subsection{Extended Ablation Study}
To further validate this trend on a larger-scale benchmark, we provide extended ablation results on CIFAR100-LT and ImageNet-LT. \cref{tab:supp-cifar100_abl,tab:supp-imagenet-abl} report FDR scores and top-1 accuracy for Many, Medium, and Few classes across various settings.

In both datasets, increasing either $\vartheta$ or $\varpi$ reduces mask sparsity, resulting in denser mask outputs and reduced node specialization. This, in turn, leads to lower FDR scores and decreased overall classification accuracy, with the model's behavior approaching performance without a gating layer.

For example, on ImageNet-LT, as $\vartheta$ increases from 0.5 to 0.9, FDR drops from $9.19\times10^2$ to $8.42\times10^2$, and average accuracy decreases from 56.90\% to 55.91\%. Increasing $\varpi$ from 0.1 to 0.9, while fixing $\vartheta$, further decreases FDR to $8.21\times10^2$ and lowers accuracy to 53.15\%.

\cref{fig:supp-heat} illustrates the hyperparameter sweep for $\vartheta$ and $\varpi$ on the CIFAR100-LT dataset. The results are consistent with the extended ablation study. Based on our experiments, we recommend the following approach for tuning $\vartheta$ and $\varpi$ on a new dataset. First, fix the slab variance $\varpi$ to a small value (e.g., 0.1). Second, search for the optimal $\vartheta$ in a range like $[0.1, 0.3, 0.5, 0.7]$. Finally, with the best $\vartheta$ fixed, perform a fine-grained search for $\varpi$ around the initial value (e.g., $[0.05, 0.1, 0.3]$).

\rev{\subsection{Mask Mode at Training and Inference}
\label{sec:supp-maskmode}
Training uses hard masks while inference uses the soft sigmoid gate. To check the impact of gate design, we conduct all four combinations on CIFAR100-LT (\cref{tab:maskmode}). The learned geometry is set by the training mode alone: hard training raises train FDR from $6.97\times10^{4}$ to $9.94\times10^{4}$ and test FDR from $1.21\times10^{2}$ to $1.28\times10^{2}$, while the inference mode leaves both unchanged. Accuracy depends on both, and our pairing is best overall; soft training costs more than two points, with the loss concentrated on Medium and Few.

\begin{table}[tb]
\setlength{\tabcolsep}{4pt}
\centering\footnotesize
\caption{\rev{Mask mode at training and inference on CIFAR100-LT. FDR is determined by the training mode alone. Bold marks the best; underline the second-best.}}
\label{tab:maskmode}
\begin{tabular}{@{}l|cc|cccc@{}}
\toprule
Method & FDR (Tr) & FDR (Te) & Many & Med. & Few & Avg. \\
\midrule
Soft / Soft & $6.97{\times}10^{4}$ & $1.21{\times}10^{2}$ & \textbf{77.53} & \underline{52.03} & 25.10 & 52.35 \\
Soft / Hard & $6.97{\times}10^{4}$ & $1.21{\times}10^{2}$ & 76.18 & 50.06 & 24.55 & 51.03 \\
Hard / Hard & $9.94{\times}10^{4}$ & $1.28{\times}10^{2}$ & \underline{75.34} & 52.01 & \underline{29.45} & \underline{53.41} \\
Hard / Soft (Ours) & $9.94{\times}10^{4}$ & $1.28{\times}10^{2}$ & 74.71 & \textbf{55.77} & \textbf{31.65} & \textbf{54.73} \\
\bottomrule
\end{tabular}
\end{table}}

\rev{\subsection{Residual Mask Structure}
\label{sec:supp-residual}

The mask similarities in \cref{tab:supp-mask} are inflated by a component shared across the entire dataset: the gate learns a common activation pattern before differentiating classes. To isolate the class-related part, we subtract the dataset-mean mask $\bar{m}=\frac{1}{N}\sum_i m_i$ and recompute cosine similarity on the residuals $m_i-\bar{m}$.

\begin{table}[tb]
\centering\footnotesize
\caption{\rev{Residual cosine similarity after removing the dataset-mean mask, CIFAR100-LT.}}
\label{tab:residual}
\begin{tabular}{@{}l|cc@{}}
\toprule
 & Intra-class & Inter-class \\
\midrule
Train & 0.986 & 0.254 \\
Test  & 0.383 & 0.015 \\
\bottomrule
\end{tabular}
\end{table}

As \cref{tab:residual} shows, the intra/inter gap is on both splits. The test-set values are much smaller: the gate is conditioned on the feature, so unseen samples inherit class structure only to the extent the feature map transfers it. The structure is thus a statistical property of the learned map, not a constraint enforced by the objective, and we claim only an empirically class-related residual structure.}

\subsection{Qualitative Result}
\cref{fig:supp-conf} shows a confusion matrix on the test set of CIFAR100-LT ($\rho=100$). Classes are sorted in descending order by the number of samples per class. In the head classes, it shows almost perfect results. In the tail classes, there are some notable misclassifications, such as `streetcar'  vs.\ `bus' and `whale' vs.\ `dolphin', where the former belong to tail classes and the latter to head classes. This error is more likely due to semantic overlap within the class itself or in the background, rather than to the long-tailed property.

\paragraph{Failure Case Analysis.}
\label{sec:supp-mask}
\cref{fig:supp-masksim} analyzes two visually similar Few/Many pairs (`streetcar'/`bus', 11/262 training samples; `whale'/`dolphin', 6/119), split into success and failure cases. For correctly classified samples, the Few-class masks are distinct from the Many counterpart (\cref{fig:supp-masksim}a), so the gate assigns them distinguishable node patterns; misclassified samples show high mask similarity to the Many class, indicating failure from visual overlap rather than long-tailed bias. Removing ETF-AM increases this similarity (\cref{fig:supp-masksim}b), so the adaptive margin supplies the separation the gate needs, consistent with the Few-class drop when ETF-AM is removed as in \cref{sec:ablation}.

\begin{figure*}[p]
  \centering
  \includegraphics[width=\linewidth]{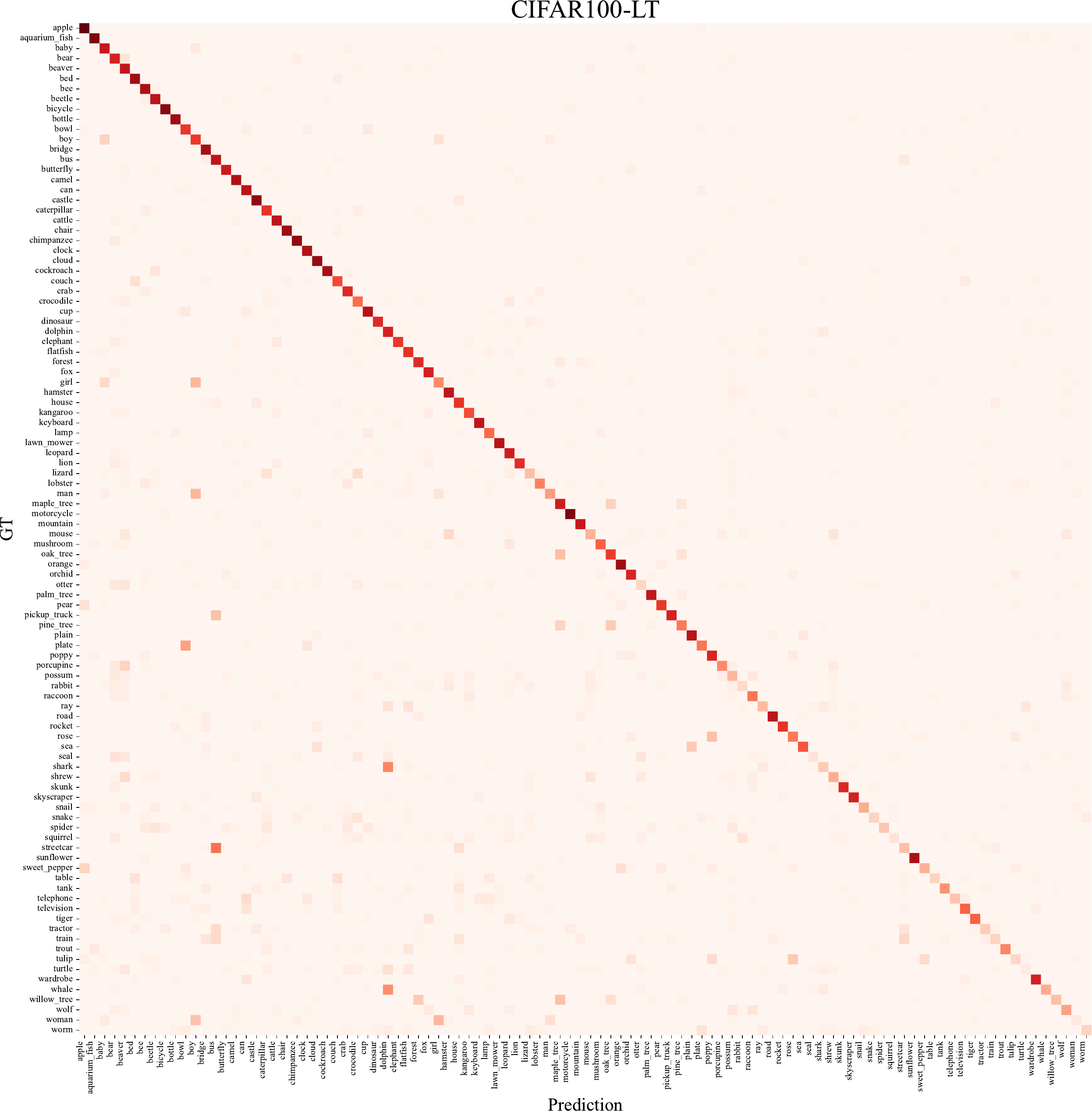}
  \caption{Confusion matrix for CIFAR100-LT ($\rho=100$). Classes are sorted in descending order by the number of training samples per class.}
  \label{fig:supp-conf}
\end{figure*}

% \clearpage

\begin{figure*}[tbh]
  \centering
  \includegraphics[width=0.4\linewidth]{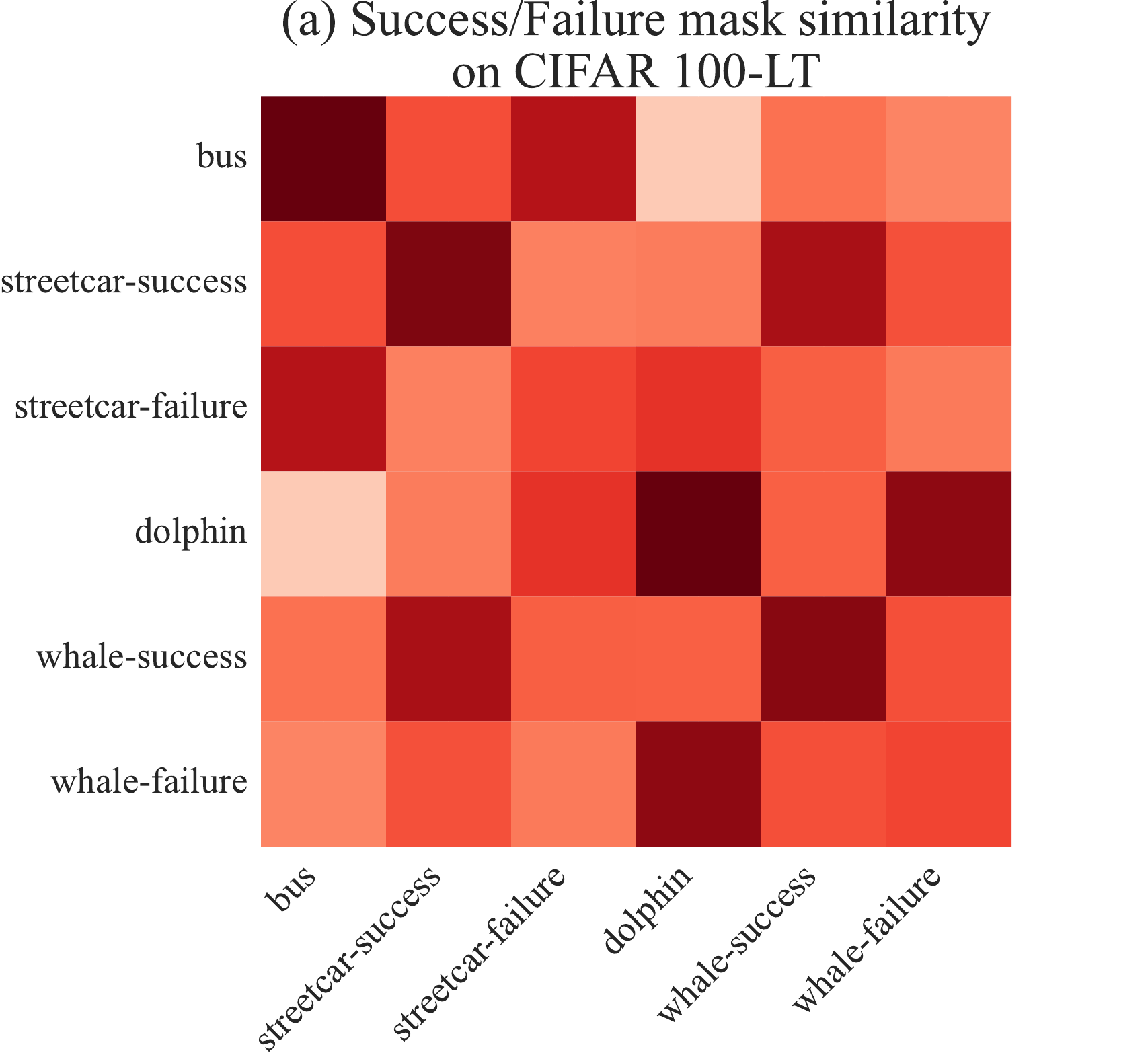}
  \includegraphics[width=0.41\linewidth]{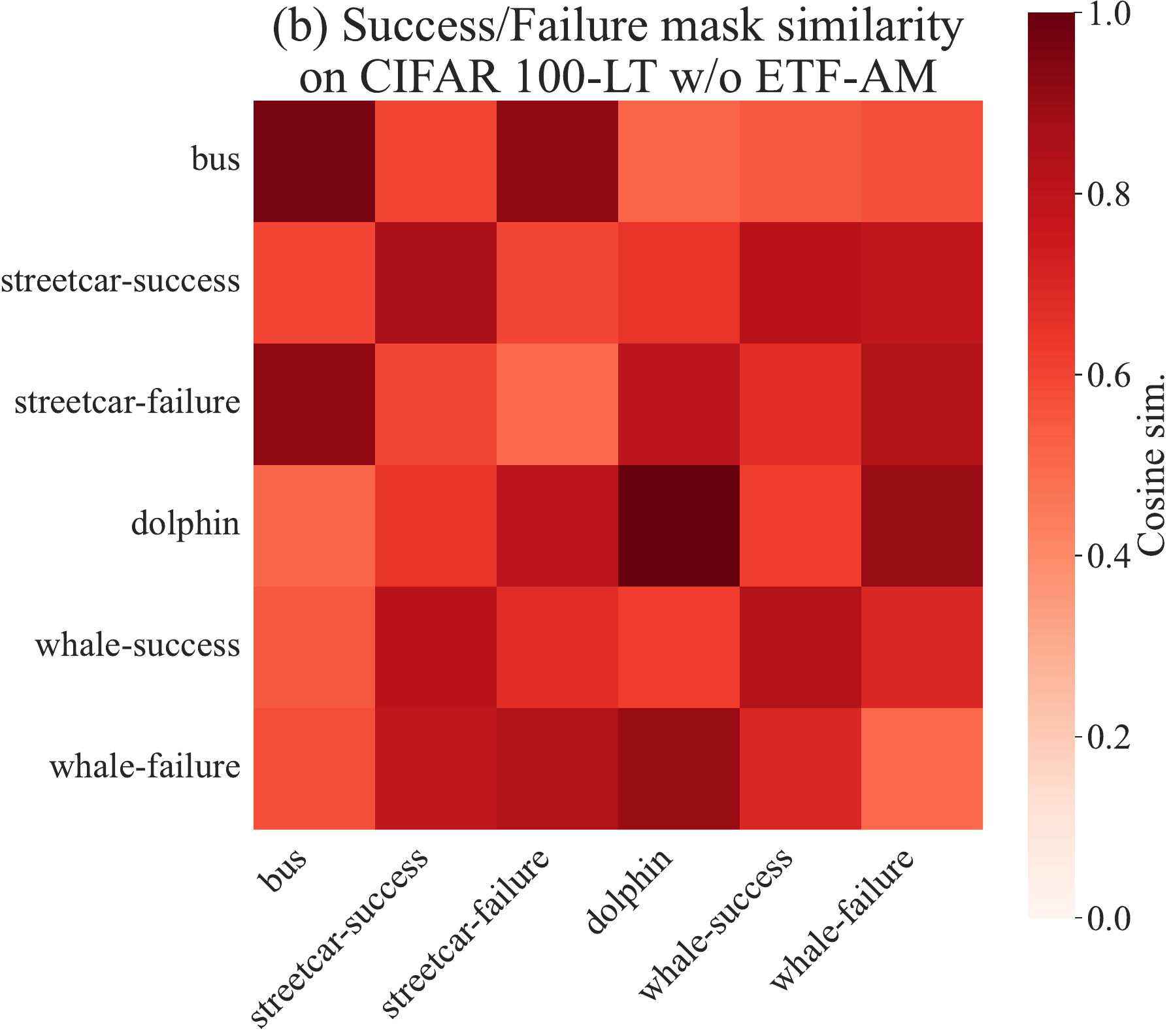}
  \caption{Generated mask similarity on the CIFAR100-LT test set for success and failure cases (`streetcar' (11 training samples) vs.\ `bus' (262) and `whale' (6) vs.\ `dolphin' (119)): (a) the proposed method, (b) the proposed method w/o ETF-AM loss.}
  \label{fig:supp-masksim}
\end{figure*}

\end{document}